\documentclass[10pt,journal,cspaper,compsoc]{IEEEtran}
\usepackage{epsfig}
\usepackage{epstopdf}
\usepackage{graphicx}
\DeclareGraphicsExtensions{.ps}
\usepackage{amsmath}
\usepackage{amssymb}
\usepackage[latin1]{inputenc}
\usepackage{subfigure}

\usepackage[section]{placeins}

\newtheorem{theorem}{Theorem}[section]

\newtheorem{lemma}[theorem]{Lemma}
\newtheorem{definition}[theorem]{Definition}

\usepackage{version}
\excludeversion{automatic}
\includeversion{manual}

\excludeversion{degenerate2}
\includeversion{degenerate1}

\excludeversion{conic}
\excludeversion{unmarginalised}

\includeversion{appendix}

\title
{Integral Geometric Dual Distributions of\\ Multilinear Models}

\author{Sami Sebastian Brandt,~\IEEEmembership{Member,~IEEE}%
\IEEEcompsocitemizethanks{\IEEEcompsocthanksitem S. S. Brandt is with the Department of Computer Science, IT University of Copenhagen, Rued Langgaards Vej 7, 2300 Copenhagen S, Denmark.
}}

\begin{document}






\newcommand{\A}{\mathbf{A}}
\newcommand{\B}{\mathbf{B}}
\newcommand{\bb}{\mathbf{b}}
\newcommand{\I}{\mathbf{I}}
\newcommand{\T}{\mathrm{T}}
\newcommand{\J}{\mathrm{J}\ }
\newcommand{\Y}{\mathbf{Y}}
\newcommand{\Z}{\mathbf{Z}}
\newcommand{\e}{\mathbf{e}}
\newcommand{\Ker}{\mathrm{Ker}\ }
\newcommand{\Ran}{\mathrm{Ran}\ }
\newcommand{\Span}{\mathrm{Span}\ }
\newcommand{\y}{\mathbf{y}}
\newcommand{\z}{\mathbf{z}}
\newcommand{\x}{\mathbf{x}}
\newcommand{\C}{\mathbf{C}}
\newcommand{\U}{\mathbf{U}}
\newcommand{\uu}{\mathbf{u}}
\newcommand{\V}{\mathbf{V}}
\newcommand{\La}{\mathbf{\Lambda}}
\newcommand{\R}{\mathbb{R}}
\newcommand{\Sbb}{\mathbb{S}}
\newcommand{\Pbb}{\mathbb{P}}
\newcommand{\rb}{\mathbf{r}}
\newcommand{\s}{\mathbf{s}}
\newcommand{\Sb}{\mathbf{S}}
\newcommand{\proj}{\mathbf{P}}
\newcommand{\Lo}{\mathbf{L}}
\newcommand{\rank}{\mathrm{rank}}
\newcommand{\dd}{\mathrm{d}}
\newcommand{\sign}{\mathrm{sign}}
\newcommand{\Jx}{\mathbf{J}}
\newcommand{\m}{\mathbf{m}}

\newcommand{\proofend}{\newline \mbox{} \hfill $\Box$}




\maketitle

\begin{abstract} 
  We propose an integral geometric approach for computing dual
  distributions for the parameter distributions of multilinear
  models. The dual distributions can be computed from, for
  example, the parameter distributions of conics, multiple view
  tensors, homographies, or as simple entities as points, lines, and
  planes. The dual distributions have analytical forms that follow
  from the asymptotic normality property of the maximum likelihood
  estimator and an application of integral transforms, fundamentally the
  generalised Radon transforms, on the probability density of the
  parameters. The approach allows us, for instance, to look
  at the uncertainty distributions in feature distributions, which are
  essentially tied to the distribution of training data, and helps us
  to derive conditional distributions for interesting variables and
  characterise confidence intervals of the estimates.
\end{abstract}

\noindent Keywords: {Duality, Radon Transform, Uncertainty, Confidence Intervals, Integral Geometry, Computer Vision.}


\section{Introduction}



An essential part of geometric computer vision are multilinear models
which include
%
%
e.g.\@ homographies, multiple view tensors, quadric surfaces, as well
as the simple entities of points, lines and planes. 
There are 
%
numerous works related to the estimation of these kinds of multilinear
relations, see e.g. \cite{Hartley00, Kanatani96,  Fischler81, Brandt06f, Agarwal05, Torr97a, Cross98, Zhang98, Triggs00, Scoleri08}. 
%
When a statistical approach is selected for the estimation, one should be able to compute the geometric model parameters along with their uncertainty distribution.  
%
%
This paper considers how these parameter distributions of multilinear
geometric entities can be dualised. 
%
%
The simplest form of this dualisation is the transformation of the
line-probability-density into a point-probability-density as proposed in
\cite{Brandt08}. This paper generalises the dualisation approach for general
multilinear models. An early, conference version of this paper is \cite{Brandt09}.

By the way of an example, consider fitting a conic section to a set of
points using maximum likelihood estimation. We would be interested in
the confidence intervals of the conic, but the normal distribution
assumption can be made only for the MLE in the parameter
space. However, we will show that by the dualisation of the parameter
distribution we will obtain a distribution of points that can be used
to plot the selected confidence intervals for the estimated conic
section. As the second example, let us consider the trifocal point
transfer. Given a point match in two views, it would be interesting to
compute the conditional position distribution of the transferred point
in the third view if we had the uncertainty information of the
trifocal tensor available. In fact, by dualising the trifocal tensor
parameter distribution, an exact form for this conditional
distribution can be computed, as will be shown later in this paper.

Our approach is closely related to the
branch of
\emph{integral geometry} in mathematics. There are two main schools of
integral geometry of which the traditional is that of Santaló and
Blaschke \cite{Santalo53}. The classical example is that the length of
a plane curve is the probability of random lines intersecting it. The
more recent meaning is the school of Gelfand
\cite{Gelfand03}. It studies integral transforms, modelled with
the Radon transform, that relates the underlying geometrical incidence
relations by incidence graphs. Our approach seems to be somewhat in
between these two schools as we compute (generalised) Radon transforms for the probability densities in such a way that the
probability measure is preserved. The dualisation is constructed from 
the fact that the probability of an
element (e.g.\@ a point) is the total probability of all the geometric
entities (e.g.\@ planes) that coincide with the element. Then by
integrating the distribution of the entity over the affine subspaces
corresponding to the selected incidence relation, Radon like integral
transforms follow.


As the principal assumption we use the normal distribution assumption for the multilinear model parameters. This is reasonable due to
the asymptotic normality property
of the maximum likelihood estimator
, i.e., due to the fact that, with certain general regularity conditions, the
distribution of the maximum likelihood parameter estimator converges
in distribution to the normal distribution with the (pseudo)inverse of
the Fisher information as the parameter covariance matrix \cite{Karr92,Kanatani96}. 
This makes
the approach taken here fundamentally different from the work in \cite{Triggs01} where an
algebraic linear system and Gaussian approximation in the feature
space were used. 
In contrast to the work in \cite{Triggs01}, 
the dual distributions considered here have analytic forms
and are exact with the assumptions above.

%
%
This paper is organised as follows. In Section \ref{sec:multilinear}, we introduce the
multilinear models considered in this paper. In Section \ref{sec:duals}, we derive
the dual distributions by 
first assuming 
a single constraint equation and then
generalise the approach for multiple constraint equations. In Section
\ref{sec:featuremappings}, we show how
interesting conditional distributions can be extracted. 
In Section \ref{sec:examples}, 
%
%
we compute confidence intervals for conics and compute the point transfer density from two views into the third view. 
Conclusions are in Section~\ref{sec:conclusions}.

%

\vspace{-1mm}

\section{Multilinear Model} \label{sec:multilinear}


We first need to define what we mean by a multilinear model. 
\begin{definition}
	A function $f: V_1 \times \cdots \times V_k \rightarrow W$, where $V_1\ldots,V_k$ and $W$ are real vector spaces, is \emph{k-linear} if it is linear in each of its $k$ arguments:
\begin{equation}
f(\ldots,\alpha \mathbf{x}+\beta \mathbf{y},\ldots)=\alpha f(\ldots,\x,\ldots)+\beta f(\ldots,\y\ldots),
\end{equation}
for all $\alpha,\beta\in\R$ and $\mathbf{x},\mathbf{y} \in V$.
\end{definition}

\begin{definition}
Let $f$ be a $(n+1)$-linear function $f: \Pbb^m \times \ldots \times \Pbb^m \times \Pbb^{N-1} \rightarrow \mathbb{R}^L$.
The \emph{multilinear model} is defined as the relation
\begin{equation}
  f(\mathbf{x}_1,\x_2,\ldots,\mathbf{x}_n;\theta)=\mathbf{0},
\end{equation}
where $\x_i$ are \emph{feature vectors} and $\theta$ is the \emph{parameter vector}, where $\x_i \in \Pbb^m$, $\theta \in \Pbb^{N-1}$ and $i=1,2,\ldots,n$. 
\end{definition}


The definition is to be taken in the general sense so that 
any of the $n$ first arguments may be repeated arbitrary many times so that, e.g., quadratic forms are included.

With a fixed $\theta$, the multilinear model defines
a multilinear equation system, where each equation is
equivalent to a linear subspace, with co-dimension one, in the space of the joint feature vector.
\begin{definition}
  The \emph{joint feature vector} is defined as the vector $\y\in\Pbb^{N-1}$, such that
\begin{equation}
  \y \hat{=} \x_1 \otimes \x_2 \otimes \ldots \otimes \x_n
\end{equation}
containing the elements of the tensor product, up to scale, where the repeating elements have been dropped;
$\hat{=}$ denotes the correspondence between the expressions,
$\otimes$ is the tensor product, and $\x_i \in \Pbb^m$,
$i=1,2,\ldots,n$. 
\end{definition} 
For the case where the feature vectors are distinct, the mapping from the feature vectors to the joint feature vector is known as \emph{Segre embedding}. 


%

Without a loss of generality, 
we assume that 
each equation $l$ of the multilinear system is in the form of 
%
%
\begin{equation} \label{eq:genmodel}
  \theta^\T \mathbf{T}_l \y \equiv \theta^\T \y_l= 0,
\end{equation}
where $\mathbf{T}_l$ is a matrix defined by the multilinear relation. For instance, the well known point and line incidence relations \cite{Hartley00,Faugeras01} characterising multiple projective views of a scene can be written in this form.

\section{Statistical Model}
Let us assume that we have the maximum likelihood estimate $\theta_0$%
, or the corresponding robust estimate \cite{Brandt06f}, 
for the parameter vector and its the covariance matrix $\C_\theta$ available, where $\theta_0$ is constrained to lie on the unit hypersphere $\Sbb^{N-1}$ in $\R^N$. 
Using the asymptotic normality property of the MLE, we assume that
$\theta \sim N(\theta_0,\mathbf{C}_\theta)$. Since $\theta_0 \in \Ker\{\C_\theta\}$, $\theta$ has the density function  

\begin{equation}
p(\theta) \propto 
\exp\left( -\frac{1}{2} \theta^\T \mathbf{C}_\theta^{\dagger} \theta  \right),
\end{equation}
where $\cdot^\dagger$ is the Moore--Penrose pseudoinverse. 
That is, we use the tangential, normal approximation for the variation of $\theta$ on the unit hypersphere at $\theta_0$.

\section{Duals of the Parameter Distributions} \label{sec:duals}

 In this section, we dualise the normal parameter uncertainty distributions starting from 
the models defined by a single constraint
equation and generalise the point--line case \cite{Brandt08} to the point--hyperplane duality (Section~\ref{sec:single}). The general case of multiple
constraint equations is considered in Section~\ref{sec:multiple}.

\subsection{Single Constraint Equation} \label{sec:single}

Let $\theta$ be a parameter vector that defines the multilinear model by a single constraint equation 
\begin{equation} \label{eq:model}
  \theta^\T \y = 0,
\end{equation}
with the joint feature vector $\mathbf{y}\in \Pbb^{N-1}$. 
%
%


\begin{figure}[t]
  \begin{center}
  \includegraphics[width=0.75\columnwidth]{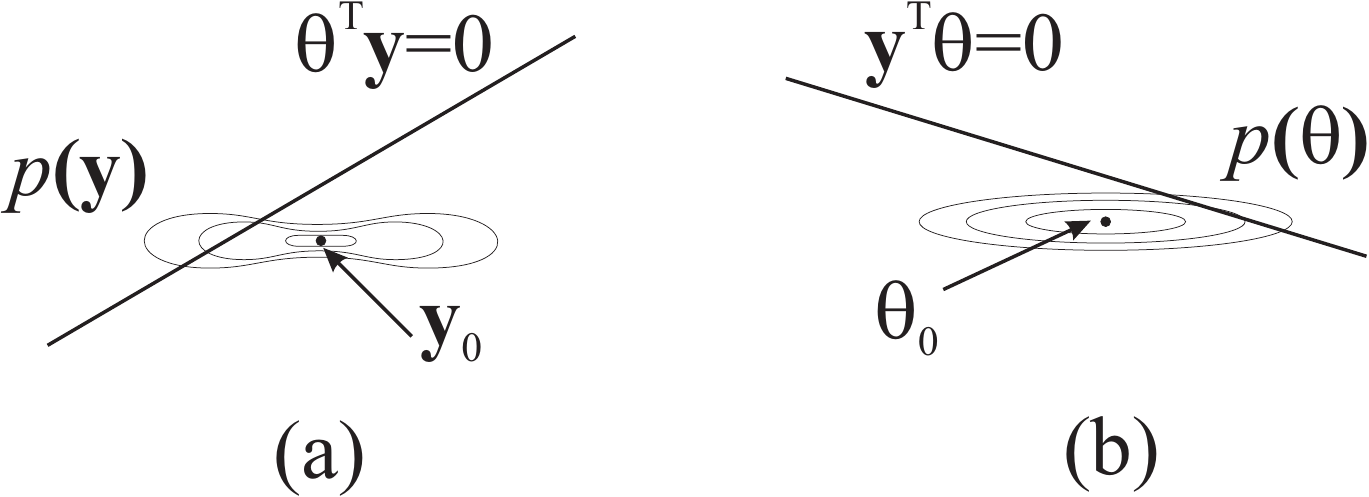}
  \end{center}
  \caption{Illustration of the duality. (a) The space of $\y$ is illustrated as the 2D plane. The parameter $\theta$ here represents the line $\theta^\T\y=0$. (b) The dual space or the space of $\theta$. In the dual space, $\y$ represents the line $\y^\T \theta=0$. The dual distribution $p(\y)$ is constructed from $p(\theta)$ by making Radon like integral transform to it so that the probability measure is preserved.
} \label{fig:duality} \vspace{-2mm}
\end{figure}

Our intention is to compute the dual of $p(\theta)$. The duality
between the parameter vector $\theta$ and the feature vector $\y$ 
is illustrated in Fig.~\ref{fig:duality}. As (\ref{eq:model})
represents the hyperplane $\theta$ in the space of $\y$, it can also
be seen as the hyperplane $\y$ in the dual space or space of
$\theta$. Now, as the dual distribution of $\theta$, which is a
distribution for $\y$, we will identify
\emph{the total probability (density) of the all the models $\theta$ lying on
the hyperplane $\y$ in the dual space}. To construct the dual pdf
$p(\y)$, we first make the whitening transformation for
variables. Then we compute the Radon transform, i.e., transform the
whitened dual domain by computing the integrals of $p(\theta)$ over
all the hyperplanes in the dual space. The subsets of these integrals
corresponding to parallel hyperplanes form the conditional probability
density (marginal density) conditioned on a fixed normal direction of
the hyperplane. Multiplying this conditional density with the pdf of
normal directions, which is the uniform density on the half of the unit
hypersphere with the whitened Gaussian model, we obtain a valid
probability density which can be interpreted in the parameter space of
$\y$. The details 
are below.

%

\begin{degenerate1}

Let us make the whitening transform from the eigenvalue decomposition of $\C_\theta$, which has been constructed so that $\theta_0$ corresponds to the last eigenvector. 
We may write 
%
%
%
\begin{equation}
\begin{split}
 \theta^\T \C_\theta^\dagger\theta&=
\theta^\T \U
\La^\dagger
\U^\T \theta\\
&=
\underset{\triangleq\theta'^\T}{\underbrace{
\theta^\T \tilde{\U}
\begin{pmatrix}
  \tilde{\La}^{-1/2} & \mathbf{0}\\
  \mathbf{0}^\T & 1 
\end{pmatrix}}}
\underset{\triangleq \tilde{\I}}{\underbrace{
\begin{pmatrix}
\I & \mathbf{0}\\
\mathbf{0}^\T & 0
\end{pmatrix}}}
\underset{\triangleq\theta'}{\underbrace{
\begin{pmatrix}
  \tilde{\La}^{-1/2} & \mathbf{0}\\
  \mathbf{0}^\T & 1
\end{pmatrix} \tilde{\U}^\T \theta}}\\
&=\theta'^\T \tilde{\I} \theta', 
\end{split}
\end{equation}
where the diagonal matrix $\tilde{\La}$ contains the $M-1$ non-zero eigenvalues of $\mathbf{C}_\theta$, sorted in the descending order, and $\tilde{\U}$ contains the corresponding eigenvectors and the eigenvector representing $\theta_0$. Now, $\theta' \sim N(\e_{M},\tilde{\I})$, where $\e_{M}$ is the standard basis vector $(0,\ldots,0,1) \in \mathbb{R}^{M}$. Furthermore,
%
for all the feature vectors that are consistent with the model $\theta$, lying on the tangent space, we may write 
\begin{equation} \label{eq:hyperplane}
\begin{split}
 0 &= \theta^\T \y 
= \underset{\theta'^\T}{\underbrace{
\theta^\T \tilde{\U}
\begin{pmatrix}
  \tilde{\La}^{-1/2} & \mathbf{0}\\
  \mathbf{0}^\T & 1 
\end{pmatrix}}}
\underset{\triangleq\y'}{\underbrace{
\begin{pmatrix}
  \tilde{\La}^{1/2} & \mathbf{0}\\
  \mathbf{0}^\T & 1
\end{pmatrix} \tilde{\U}^\T \y}}
= \theta'^\T \y', 
\end{split}
\end{equation}
where $\y'\in \Pbb^{M-1}$ represents a \emph{reduced joint feature vector}. 
%
%
%
For simplicity, we now investigate the reduced, transformed model $\theta'$.  
Now we are ready to state our main theorem for the case of the single constraint equation.

\begin{theorem} {\it
   Let $\y,\theta \in \mathbb{P}^{N-1}$, normalised so that the joint feature vector $\y$ has
   homogeneous scaling of unity and $\theta$ lies in the tangent space
   $T_{\theta_0}(\mathbb{S}^{N-1})$. Moreover, let
   $\theta \sim N(\theta_0,\C)$ with the probability density function $p(\theta)$, where the eigenvectors of $\C$
   corresponding to the $M-1$ non-zero eigenvalues span $T_{\theta_0}$, and $\y'$ is the reduced joint feature vector corresponding to $\y$. The dual distribution of $p(\theta)$ has the analytic form 
\begin{equation}
\begin{split}
p(\rho,\phi)&=
\frac{\Gamma(M/2-1/2) \mathrm{e}^{-\frac{1}{2}\rho^{-2}}\prod_{i=1}^{M-3} \sin^{M-2-i}(\phi_i)}{\sqrt{2\pi^{M}}\rho^2},
\end{split}
\end{equation}
where the reduced joint feature vector $\y'$ is parameterized by the modified hyperspherical coordinates $(\rho,\phi)$ in $\R^{M-1}$.
}
\end{theorem} 

\begin{proof}
On the basis of the construction above, the variation of $\theta'$ occurs
only in the $M-1$ dimensional affine subspace $\pi'$ perpendicular to
$\e_M$ in the dual space.
Since every point on this tangent hyperplane can be identified with a
unique one-dimensional linear subspace of $\R^M$, we may regard the
tangent hyperplane as a projective space $\Pbb^{M-1}$. Moreover, the
points on the tangent plane can be considered to be already in the
homogeneous form. Hence, in the following we assume that $\theta' \in
\Pbb^{M-1}$.

\begin{figure}
  \begin{center}
  \includegraphics[width=0.75\columnwidth]{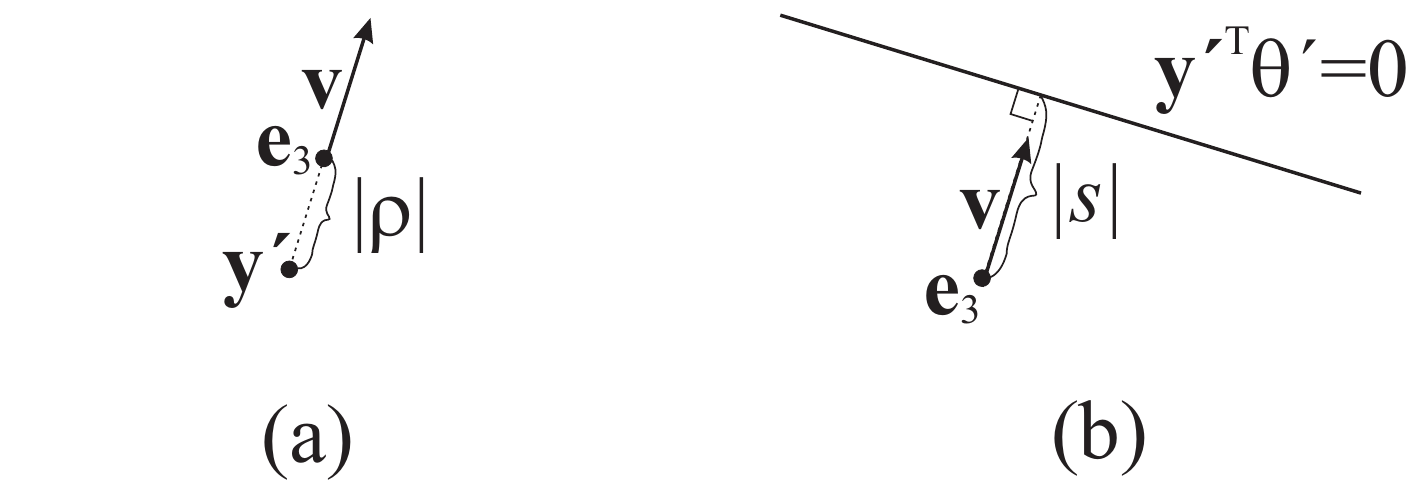} \vspace{-1.5mm}
  \end{center}  
  \caption{The hyperplane parameterisation when $M=3$, thus, the hyperplane $\y'$ is a line and the tangent affine subspace $\pi'$ is the real plane. (a) Reduced joint feature space, where $\y'$, represented by the homogeneous vector $(\rho \mathbf{v}^\T\ 1)^\T$, is a point; (b) the corresponding dual space, where $\y'$ is the line. 
} \label{fig:dualplane}
\end{figure}

In the dual space, the $M-2$ dimensional hyperplanes $\y'^\T\theta'=0$
embedded in the $M-1$ dimensional affine subspace $\pi'\hat{=}\mathbb{R}^{M-1}$, may be parameterised by the signed distance $s$ from the
origin and by the unit vector $\mathbf{v} \in \Sbb^{M-2}$, as Fig~\ref{fig:dualplane}(b) illustrates. The origin can be
represented by $\e_M\in \mathbb{P}^{M-1}$ 
and we assume that $\mathbf{v}=\mathbf{v}(\phi)$ where
$\phi$ parameterises 
the normal direction of the $M-2$ dimensional hyperplane $\y'$ in 
$\mathbb{R}^{M-1}$.  We choose the sign
of $s$ to be equal to sign of the intercept of the hyperplane in the
dual space, hence, 
\begin{equation}
s= -\frac{y'_N}{ \mathrm{sign}\ ( y'_{N-1} ) \sqrt{\y'^\T \y' - {y'}_N^2}}.
\end{equation}
On the other hand, in the reduced joint feature space, the homogeneous vector $\y'$ can be parameterised
by the signed distance $\rho$ from the origin $\e_M$ and the direction $\mathbf{v}$ (Fig~\ref{fig:dualplane}(a)).
 The sign of $\rho$ is identified as the last variable sign of the inhomogeneous representation of $\y'$. 
Then we have 
%
\begin{equation}
  s=-\frac{1}{\rho}.
\end{equation}

As we assume that a $\theta' \sim N(\e_M,\tilde{\I})$, the probability of the hyper plane ${\y'_s}^\T \theta'=0$, conditioned on the direction $\mathbf{v}(\phi)$ in the dual space, is simply the marginal probability of the Gaussian over the hyperplane, i.e., 
\begin{equation} \label{eq:1Dgaussian}
  p(s|\phi)=\int_{{\y'_s}^\T \theta'=0} p_\mathrm{G} (\theta';\e_M,\tilde{\I}) \mathrm{d} S = p(s)
\end{equation}
which is a mean zero, 1-D Gaussian with unity variance \cite{Brandt08}, where $\mathrm{d} S $ denotes the volume differential in the hyperplane. 

By using the modified spherical coordinates 
$\y'=\y'(\rho,\phi)$ in the reduced joint feature space 
\begin{automatic}
(see Appendix~\ref{app:ndim}) 
\end{automatic}
\begin{manual}
(see Appendix~A) 
\end{manual}
and the fact that $s=s(\rho)$, 
we get  
%
\begin{equation}
\begin{split}
  p(\rho|\phi)&=\left|\frac{\partial s}{\partial \rho}\right|p\left(s(\rho)|\phi \right)
=\frac{1}{\sqrt{2\pi}\rho^2} \exp\left( -\frac{1}{2} \rho^{-2}  \right).
\end{split}
\end{equation}

On the other hand, since we have an isotropic Gaussian distribution, the distribution of normal directions is uniform over a half of the unit sphere. By using the modified spherical coordinates, and marginalising the $M-1$-dimensional Gaussian over the signed radial parameter, we obtain 
\begin{equation}
\begin{split}
  p(\phi)
=\frac{\Gamma\left(M/2-1/2\right)}{\pi^{(M-1)/2}}\prod_{i=1}^{M-3} \sin^{M-2-i}(\phi_i).
\end{split}
\end{equation}
%
%
Hence, 
\begin{equation} \label{eq:dualdensity}
\begin{split}
p(\rho,\phi)&=p(\rho|\phi)p(\phi),
\end{split}
\end{equation}
and the claim follows. 
\end{proof}

We have thus derived an analytic form for the probability density of the reduced joint feature vector $\y'=\y'(\rho,\phi)$ 
assuming the model (\ref{eq:model}) and the Gaussian distribution for the parameter vector $\theta$.


%
%

%


\end{degenerate1}

\subsection{Multiple Constraint Equations} \label{sec:multiple}

Now we generalise the computations above for multiple constraint equations.
Let $\theta$ be a parameter vector of a multilinear model. We define that the set vectors $\y_l,\ l=1,\ldots,L$ are consistent with the model if and only if  
%
\begin{equation}
\theta^\T \y_l =0, \quad l=1,\ldots,L. \label{eq:model2}
\end{equation}
Hence, $(\y_1\ \dots\ \y_L)^\T \theta = \Y^\T \theta = \mathbf{0}$, i.e., $\theta \in \Ker\{\Y^\T\}$ (c.f. the nullspace representation of lines in $\R^3$). Similarly as in (\ref{eq:hyperplane}), we define the transformed model $\theta'$ and the reduced coefficient matrix $\Y'$ so that 
%
\begin{equation} \label{eq:hyperplane2}
\begin{split}
 \mathbf{0} &= \theta^\T \Y 
= \underset{\theta'^\T}{\underbrace{
\theta^\T \tilde{\U}
\begin{pmatrix}
  \tilde{\La}^{-1/2} & \mathbf{0}\\
  \mathbf{0}^\T & 1 
\end{pmatrix}}}
\underset{\triangleq\Y'}{\underbrace{
\begin{pmatrix}
  \tilde{\La}^{1/2} & \mathbf{0}\\
  \mathbf{0}^\T & 1
\end{pmatrix} \tilde{\U}^\T \Y}}\\
&
= \theta'^\T \Y'.
\end{split}
\end{equation}
Let $W$ be the affine subspace corresponding to the set $\{ \theta' \in \Pbb^{M-1} | \Y'^\T \theta' =0\}$ or $W=\{\tilde{\theta}' \in \R^{M-1} | \A \tilde{\theta}' = \bb\}$, where $\theta'\hat{=}(\tilde{\theta}',1)$ and ${\Y'}^\T=(\A\ \ -\bb)$.
%
%
%
%
Assuming that the $W$ contains finite points, it is easy to show that the nearest point 
to the origin $\e_M$ is 
%
\begin{equation}
\theta'_\mathrm{min}\triangleq
\begin{pmatrix} \mathbf{w}\\
1
\end{pmatrix}
=
  \begin{pmatrix} \A^\dagger \bb\\
1
\end{pmatrix}.
\end{equation}
In other words, the affine subspace $W=\mathbf{w} + V$ where $V$ denotes the linear subspace parallel to $W$ and $\dim (V) \equiv K$.

We are interested in the probability density of $\Ker\{\Y'^\T\}$, that
is, \emph{the total probability (density) of the affine subspace $W$
as it represents the set of all the models $\theta'$ that
are consistent with $\Y'$}. To construct this dual density, instead of
integrating over all hyperplanes as in the previous subsection, we
must generalise the Radon transform by integrating $p(\theta')$ over
all the $K$-dimensional affine subspaces.  To perform
this kind of integral transform and to interpret its result as probability densities, we must create a unique
parameterisation for affine subspaces. 
A unique parameterisation could be constructed by the Grassmannian coordinates
\cite{Triggs95,Triggs01}---it could give certain algebraic benefits that are to be investigated in future. 

In this paper, we study Gaussians in the reduced coordinate frame that suggests a simple parameterisation, generalising the result on the single coordinate equation. 
%
We
represent an affine subspace by the offset vector $\mathbf{w}$ and parameters of the parallel linear subspace
$V$. To parameterise $V$, we use the fact that the related
\emph{orthogonal projection matrix} 
\begin{equation}
  \proj=\I-\A^\dagger \A, 
\end{equation} 
projecting onto the subspace is unique. It is well known that a matrix
represents an orthogonal projection onto a linear subspace if and
only if it is \emph{idempotent}, $\proj^2=\proj$, and
\emph{self-adjoint}, $\proj^\T=\proj$. 
Now, to uniquely parameterise the matrix $\proj$ we need the following lemma.
\begin{lemma} {\it
  Let $\proj$ be an idempotent and symmetric $(M-1)\times (M-1)$
  matrix. If $\mathbf{p}_i=\proj \e_i$,
  $i=1,\ldots,K$ are linearly independent and span the range of
  $\proj$, where $K\leq \frac{M}{2}$, is the dimension of the range, then
  there is a unique lower triangular $K\times(M-1)$ matrix $\Lo$ with
  orthonormal columns and strictly positive diagonal so that $\proj = \Lo \Lo^\T$. 
}
\end{lemma}
The proof is in 
\begin{automatic}
Appendix~\ref{sec:lemmaproof}.
\end{automatic}
\begin{manual}
Appendix~B.
\end{manual}

%
%
%
%
The lemma suggests that we may parameterise the orthogonal projection matrix $\proj$ by parameterising the elements of $\Lo=\Lo(\Phi)$ 
when we may define 
\begin{equation}
  \proj(\Phi)=\left\{ 
\begin{tabular}{ll} $\Lo(\Phi)\Lo(\Phi)^\T$, \quad & if $K \leq M/2$ \\
                  $\mathbf{I}-\Lo(\Phi)\Lo(\Phi)^\T$, \quad & otherwise.
\end{tabular}\right.
\end{equation}
The matrix $\Lo$ can be parameterised 
by the modified 
spherical coordinates 
\begin{automatic}
(see Appendix~\ref{app:ndim}) 
\end{automatic}
\begin{manual}
(see Appendix~A) 
\end{manual}
with the radial parameter equal to unity. 
The first column vector is on the half of the unit sphere $\Sbb^{M-2}$, the second column vector is orthogonal to the
first and it has one element less and is hence 
on $\Sbb^{M-4}$. To parameterise $\Sbb^{M-4}$, we form an orthogonal
basis in the orthogonal complement of the first column vector, by
creating the orthogonal projection matrix that projects onto the
orthogonal complement, and employ the Gram--Schmidt
orthonormalisation procedure, as shown in Appendix~B, 
and form the unit sphere in the subspace spanned by the orthonormal vectors. The other columns $l \leq \tilde{K}$ can be similarly parameterised
on the half of the unit sphere $\Sbb^{M-2l}$, whereas $\tilde{K}=K$, if $K \leq M/2$, and $\tilde{K}=M-K-1$, otherwise. Let $\Phi_k$ denote the
vector of modified spherical coordinates of the column $k$ in $\Lo$. We
collect the parameters in the vector
$\Phi=(\Phi_1,\Phi_2,\ldots,\Phi_{\tilde{K}})$. In total, there are
$K'=(M-1)\tilde{K}-\tilde{K}^2$ free parameters in $\Lo$.

Now, we are ready to state our theorem about the probability density for the affine subspace $W$. 
\begin{theorem}
{\it
   Let $\Y\in\mathbb{R}^{N \times L}$ and $\theta \in \mathbb{P}^{N-1}$, normalised so that $\theta$ lies in the tangent space
   $T_{\theta_0}(\mathbb{S}^{N-1})$. Moreover, let
   $\theta \sim N(\theta_0,\C)$ with the probability density function $p(\theta)$, where the eigenvectors of $\C$
   corresponding to the $M-1$ non-zero eigenvalues span $T_{\theta_0}$, and $\Y'$ is the reduced matrix corresponding to $\Y$. The dual distribution of $p(\theta)$, corresponding to the affine subspace $W\hat{=}\Ker\{\Y'^\T\}$, has the analytic form 
\begin{equation}
\begin{split}
p(\mathbf{s},\Phi)
&
= \frac{\mathrm{e}^{-\frac{1}{2}\mathbf{s}^\T\mathbf{s}}}{2^{(M-K-1)/2}\pi^{(M-K+\tilde{K}M-\tilde{K}^2-1)/2}} \\ & 
\times \prod_{k=1}^{\tilde{K}}\Gamma\left((M-2k+1)/2\right)\prod_{i=1}^{M-2k-1} \sin^{M-2k-i}(\phi_i^k).
\end{split}
\end{equation}
where $(\mathbf{s},\Phi)$ are the parameters of the affine subspace $W$. 
}
\end{theorem}

%

\begin{proof}
We create an orthogonal basis $\mathbf{u}_1,\mathbf{u_2},\ldots,\mathbf{u}_{M-1}$ for the subspace $V$ and its orthogonal complement $V^\perp$ using the Gram-Schmidt orthonormalisation procedure for the projections $\mathbf{p}_i=\proj \mathbf{e}_i$, $i=1,2,\ldots,K$ onto the subspace $V$ (see Appendix~B) 
and similarly for the projections $\mathbf{p}_i=(\I-\proj) \mathbf{e}_i$, $i=K+1,K+2,\ldots,M-1$ on its orthogonal complement $V^\perp$. 
We marginalise the $M-1$ dimensional Gaussian over all the parallel affine subspaces of $V$, i.e., those which are of the form $W=\mathbf{w}+V$. We obtain the
conditional probability density of
$\mathbf{w}\hat{=}\mathbf{s}\in \R^{M-K-1}$ 
\begin{equation}
\begin{split}
  p(\mathbf{s}|\Phi)
=&\int_{V
} p_\mathrm{G} (\theta';\e_M,\tilde{\I}) \mathrm{d} S \\ \mbox{}=& p(\mathbf{s}),
\end{split}
\end{equation}
which is a mean zero, $M-K-1$ dimensional Gaussian with the identity matrix as the covariance matrix. 

As to the distribution of directions $\Phi$, isotropic Gaussians imply uniform directions on half hyperspheres. We decompose the vector as $\Phi=(\Phi_1,\Phi_2,\ldots,\Phi_{\tilde{K}})$ and further $\Phi_k=(\phi_1^k,\phi_2^k,\ldots,\phi_{M-2k}^k)$. Thus we have 
\begin{equation}
\begin{split}
p(\Phi)&=p(\Phi_1,\Phi_2,\ldots,\Phi_{\tilde{K}})\\
&
=p(\Phi_1)p(\Phi_2|\Phi_1)p(\Phi_3|\Phi_1,\Phi_2) \cdots 
\\&\quad \quad \quad 
p(\Phi_{\tilde{K}}|\Phi_1,\Phi_2,\ldots,\Phi_{\tilde{K}-1}),
\end{split}
\end{equation}
where 
\begin{equation}
\begin{split}
  p(\Phi_k|&\Phi_1,\Phi_2,\ldots,\Phi_{k-1})\\ 
&
=\frac{\Gamma\left(\left(M-2k+1\right)/2\right)}{\pi^{(M-2k+1)/2}}\prod_{i=1}^{M-2k-1} \sin^{M-2k-i}(\phi_i^k),
\end{split}
\end{equation}
$k=1,2,\ldots,\tilde{K}$.

So finally, 
\begin{equation} \label{eq:pW}
\begin{split}
  p(W)&=p(\mathbf{s},\Phi)=p(\mathbf{s}|\Phi)p(\Phi), 
\end{split}
\end{equation}
and the claim follows. 
\end{proof}

We have thus derived an analytic form for the probability density of the affine
subspace $W=W(\mathbf{s},\Phi)$ or, equivalently, the probability
density of the left nullspace of the reduced 
coefficient matrix $\mathbf{Y}'$,
assuming the model (\ref{eq:model2}) and Gaussian distributed
parameter vector $\theta$.

\begin{unmarginalised}


Let $\mathbf{Z}$ be a $M\times K$ matrix containing linearly independent column vectors $\mathbf{z}_k$, $k=1,\ldots,K$, where 
$K=\mathrm{rank}(\Y')$ and $\Ker\{\mathbf{Z}^\T\}=\Ker\{\Y'^\T\}$.  
In other words, $W_0$ is the
intersection of the hyperplanes \mbox{$\z_k^\T\theta'=\mathbf{0}$}, $k=1,2,\ldots,K$ in the
dual space, where we parameterise the hyperplanes by their signed
distance $\s=(s_1,s_2,\ldots,s_K)$ from the origin and their unit
normal vectors $\mathbf{v}_k=\mathbf{v}(\varphi^k)$ using the modified
spherical coordinates $\varphi^k$, $k=1,2,\ldots,K$, 
$\Phi=(\varphi^1,\varphi^2,\ldots,\varphi^K)$, as described in the
previous section. 
%
%
%

The subspace $Z$ spanned by $\{\z_1,\ldots,\z_K\}$ is orthogonal to any
affine subspace $W$ parallel to $W_0$ in $\R^{M-1}$. These
affine subspaces can be hence parameterised by the $K$ dimensional
vector $\mathbf{\mathbf{w}}$ representing the point of intersection between $W$ and $Z$. We define $\mathbf{\mathbf{w}}$ in
the orthonormal basis
$B=\{\mathbf{u}_1,\ldots,\mathbf{u}_K\}$, obtained from the vectors
$\mathbf{v}_1,\ldots,\mathbf{v}_K$ by the Gram--Schmidt
orthonormalisation procedure. The conditional probability density of
$\mathbf{w}$ is obtained as
%
\begin{equation}
\begin{split}
  p(\mathbf{w}|\mathbf{v}_1,\mathbf{v}_2,\ldots,\mathbf{v}_K)
=&\int_{W
} p_\mathrm{G} (\theta';\e_M,\tilde{\I}) \mathrm{d} S \\ \mbox{}=& p(\mathbf{w}),
\end{split}
\end{equation}
which is a mean zero, $K$ dimensional Gaussian with the identity matrix as the covariance matrix. 
For the Gaussian, we make the coordinate transform $\mathbf{w}=\mathbf{w}(\mathbf{s})$, where
\begin{equation}
\mathbf{s}=\tilde{\mathbf{V}}^\T \mathbf{w} \Leftrightarrow \mathbf{w}=\tilde{\mathbf{V}}^{-\T}\mathbf{s},
\end{equation}
whereas $\tilde{\mathbf{V}}=(\tilde{\mathbf{v}}_1\ \ldots\ \tilde{\mathbf{v}}_K)$ represent the direction vectors $\mathbf{v}_1,\ldots,\mathbf{v}_K$ in the basis $B$. Hence,
\begin{equation}
\begin{split}
  p(\mathbf{s}|\mathbf{v}_1,\mathbf{v}_2,\ldots,\mathbf{v}_K)
&=| \det \J \mathbf{w}(\mathbf{s}) |  
  p(\mathbf{w(\mathbf{s})}|\mathbf{v}_1,\mathbf{v}_2,\ldots,\mathbf{v}_K)\\
&= \frac{\mathrm{e}^{-\frac{1}{2}\| \tilde{\mathbf{V}}^{-\T} \mathbf{s} \|^2}}{(2 \pi)^{K/2} |\det(\tilde{\mathbf{V}})|}. 
\end{split}
\end{equation}

On the other hand, since the Gaussian distribution is isotropic, the distribution of directions is uniform with the density 
\begin{equation}
\begin{split}
p(\mathbf{v}_1,\mathbf{v}_2,\ldots,\mathbf{v}_K)
&=\frac{\Gamma(M/2-1/2)^{K}}{\pi^{K(M-1)/2}}. 
\end{split}
\end{equation}
Hence,
\begin{equation}
\begin{split}
p(\mathbf{s},\mathbf{v}_1,&\mathbf{v}_2,\ldots,\mathbf{v}_K)\\
=&p(\mathbf{s}|\mathbf{v}_1,\mathbf{v}_2,\ldots,\mathbf{v}_K)p(\mathbf{v}_1,\mathbf{v}_2,\ldots,\mathbf{v}_K)\\
=&\frac{
\Gamma(M/2-1/2)^{K}
\mathrm{e}^{-\frac{1}{2}\| \tilde{\mathbf{V}}^{-\T} \mathbf{s} \|^2}
}{\sqrt{2^{K}\pi^{KM}}|\det(\tilde{\mathbf{V}})|}. 
\end{split}
\end{equation}

Again, we use the modified spherical coordinates to parameterise the the reduced joint feature vectors, i.e., $\Z=\Z(\rb,\Phi)$, where $\rb=(r_1,\ldots,r_K)$ and $\Phi=(\varphi^1,\ldots,\varphi^K)$. In addition, we use the fact that $\s=\s(\rb)$ and $\mathbf{v}_k=\mathbf{v}_k(\varphi^k)$, $k=1,\ldots,K$. Now, we get 
\begin{equation}
\begin{split}
  p(&\rb,\Phi)\\&=\left|\det\frac{\partial(\mathbf{s},\mathbf{v}_1,\ldots,\mathbf{v}_K)}{\partial(\mathbf{r}, \Phi)}\right|p\left(\mathbf{s}(\mathbf{r}),\mathbf{v}_1(\varphi^1),\ldots,\mathbf{v}_K(\varphi^K)\right)\\
&=\frac{
\Gamma(M/2-1/2)^{K}
\mathrm{e}^{-\frac{1}{2}\| \tilde{\mathbf{V}}^{-\T} \mathbf{r}_{-1} \|^2}
}{\sqrt{2^{K}\pi^{KM}}|\det(\tilde{\mathbf{V}})|}\\
&\quad \quad \times
\prod_{k=1}^K\frac{\prod_{i=1}^{M-3} \sin^{M-2-i}(\varphi^k_i)}{r_k^2}
,
\end{split}
\end{equation}
where $\mathbf{r}_{-1}=(1/r_1,1/r_2,\ldots,1/r_K)$. 

Now, we are interested in the total probability of $\Ker\{\Y'^\T\}$
but there are infinitely many matrices that have the same
kernel. 
%
Therefore, we need
to marginalise over all those $\mathbf{Z}$ that has the kernel equal
to $\Ker\{\Y'^\T\}$. Hence,
\begin{equation} \ref{eq:pW}
\begin{split}
p(\Ker\{\Y'^\T\})&=\int_{\Ker\{\Z^\T\}=\Ker\{ \Y'^\T \}} p(\Z) \mathrm{d} \Z\\
&=\int_\Xi p(\Z_0 \Xi) | \det \J (\mathbf{r},\Phi)(\Xi) | \mathrm{d}\Xi,
\end{split}
\end{equation}
where $\Z_0$ is any fixed $\mathbf{Z}$.

This is an analytic form for the probability density for the kernel of the reduced joint feature matrix $\Y'$ 
assuming the model (\ref{eq:model2}) and the Gaussian distribution for the parameter vector $\theta$.

\end{unmarginalised}

\section{Mappings to Feature Distributions} \label{sec:featuremappings}



 In the previous section, we derived the general form for the dual
 distributions in the function of the parameters of the corresponding
 affine subspace. Now we discuss how we can extract interesting
 feature distributions from the dual distributions. 
The interesting feature distribution often has less parameters than
the affine subspace or 
one may be interested only in certain conditional
feature distributions, conditioned on some fixed 
a subset
of the features. In these cases the mapping from the feature
distribution to the affine subspace $W$ is not necessarily one-to-one
and onto. However, we may form the conditional distribution
\begin{equation} \label{eq:conditionald}
  p(\mathbf{s},\Phi|R)=\frac{p(\mathbf{s},\Phi)}{\iint_{R}
  p(\mathbf{s},\Phi) \dd \mathbf{s} \dd \Phi} \propto
  p(\mathbf{s},\Phi).
\end{equation}
conditioned on the restriction $R\subset \Omega$, where $R$ is
parameterised by the interesting part $\tilde{\x}\in \R^{\tilde{N}}$
of the feature distribution.
%

 \begin{theorem} {\it
    Let us consider a smooth submanifold $R$ of the parameter space $\Omega$ so that there is a differentiable bijective map between the parameters $\tilde{\mathbf{x}}\in X \subset \mathbf{R}^{\tilde{N}}$ and $(\mathbf{s},\Phi)\in R$ almost everywhere. Then
\begin{equation}
  p(\tilde{\x}|X) = \frac{\sqrt{\det(\Jx^\T \Jx)}
p(\mathbf{s}(\tilde{\x}),\Phi(\tilde{\x}))}
{\iint_R \sqrt{\det(\Jx^\T \Jx)}
p(\mathbf{s}(\tilde{\x}),\Phi(\tilde{\x})) \mathrm{\mathrm{d}}\mathbf{s} \mathrm{\mathrm{d}}\Phi },
\end{equation}
where $\mathbf{J}=\frac{\partial (\mathbf{s},\Phi)}{\partial \tilde{\mathbf{x}} }$.
}
  \end{theorem}
 \begin{proof}
    Let $\varphi:$ $X \rightarrow R$ so that $(\mathbf{s},\Phi)=\varphi(\tilde{\x})$ is continuous and invertible almost everywhere on $X\subset\R^{\tilde{N}}$.     
%
Given an orthonormal basis $B=\{\uu_1,\uu_2,\ldots,\uu_{\tilde{N}}\}$ on the tangent space $T_{\varphi_0}(R)$ 
\begin{equation}
  \varphi - \varphi_0 = \left(\uu_1\ \uu_2\ \cdots\ \uu_{\tilde{N}} \right)^\T 
(\xi_1\ \xi_2\ \cdots\ \xi_{\tilde{N}})^\T,
\end{equation}
where $\varphi_0=\varphi(\tilde{\mathbf{x}}_0)$ and $\varphi \in T_{\varphi_0}(R)$. 
The Jacobian of the mapping $\tilde{\x}\mapsto \xi$ is the $\tilde{N}\times\tilde{N}$ matrix 

\begin{equation}
  \Jx_\xi=\frac{\partial{\xi}}{\partial \tilde{x}}=\frac{\partial{\xi}}{\partial \varphi}\Jx_0=\left(\uu_1\ \uu_2\ \cdots\ \uu_{\tilde{N}} \right)^\T \Jx_0,
\end{equation}
where we have used the property $\xi_i=\uu_i^\T (\varphi-\varphi_0)$, $i=1,\ldots,\tilde{N}$. Thus
\begin{equation}
\begin{split}
  | \det \Jx_\xi |&= \sqrt{| \det \Jx_\xi |^2}
= \sqrt{\det (\Jx_0^\T \Jx_0)},
\end{split}
\end{equation}
which holds almost everywhere and is independent of the choice of the orthonormal bases. 
%
Using the substitution rule for integrals, the
conditional distribution of 
$\tilde{\x}$
takes the form 
\begin{equation} \label{eq:conditional}
p(\tilde{\x}|X) \propto \sqrt{\det(\Jx^\T \Jx)}
p(\mathbf{s}(\tilde{\x}),\Phi(\tilde{\x})).
\end{equation} 
\end{proof}

As $p(\tilde{{\x}}|X)$ is easily computed up to a global constant,
we may draw samples from it by generating a
MCMC chain with the Metropolis--Hastings sampling rule. Moreover, at least in cases where the interesting distribution covers a linear manifold, direct sampling methods can be applied similar to those derived in \cite{Brandt08}. 

\section{Experiments} \label{sec:examples}

 In this section we show two application examples of the dual distributions. We create confidence intervals for conics (Section~\ref{sec:conics}) and show how probabilistic point transfer can be constructed by using the covariance information of the estimated trifocal tensor (Section~\ref{sec:trifocaltensor}).

\subsection{Dual Distributions for Conics} \label{sec:conics}

%
 The points on a conic satisfy the homogeneous quadratic equation 
\begin{equation}
  \mathbf{x}^\T \mathbf{A} \mathbf{x} = 0,
\end{equation}
which is a bilinear equation in $\mathbf{x}$ and $\mathbf{A}$ is a symmetric $3\times 3$ matrix. 
This equation can be written in the form, using the \emph{Veronese embedding},  
\begin{equation} \label{eq:linearform}
\theta^\T \mathbf{y}=0,
\end{equation}
where $\theta \hat{=}(a_{11},a_{22},a_{33},2 a_{12},2 a_{23},2 a_{13})$ and $\mathbf{y}\hat{=}(x_1^2, x_2^2, x_3^2,x_1 x_2, x_2 x_3, x_1 x_3)$,
%
i.e., a conic forms a 5-dimensional linear subspace in the six dimensional joint feature space. We assume that the parameter vector estimate $\theta=\theta_0$ and its covariance matrix $\mathbf{C}_\theta$ are available.

If we dualise the relationship (\ref{eq:linearform}) above, we see that a fixed point on the
image plane determines a 5-dimensional linear subspace in the dual
space, i.e., the space of all conics that intersect the point on
the image plane, see Fig.~\ref{fig:conicduality}. Moreover, to characterise
the probability of the point on the image we may construct the \emph{total
probability (density) of all those conics containing the point} by using the
uncertainty distribution of the estimated conic. 
%
In other words, the dual distribution 
characterises the
confidence of the estimated conic 
by illustrating
 what has been learned from the locations of the points on the true conic.  

\begin{figure}[t]
\begin{center}
  \includegraphics[width=0.75\columnwidth]{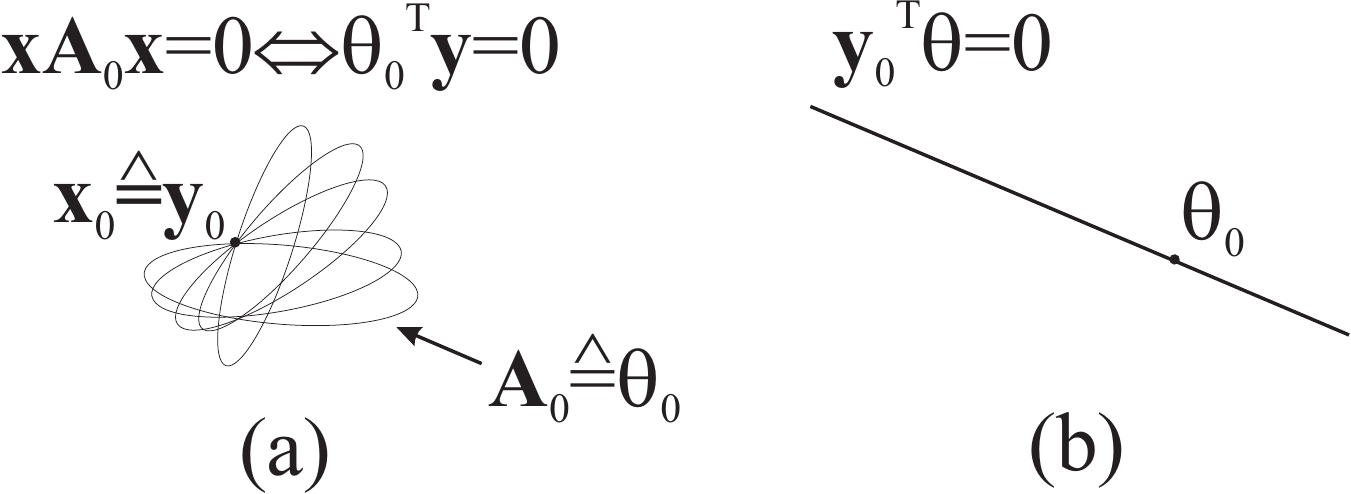}
\end{center} 
  \caption{(a) The dual distribution as the point distribution is here generated from the total probability of all the conics containing the point. (b) The subset of conics containing a given point is a five-dimensional linear subspace in the six-dimensional dual space, thus the probability of a point can be identified as the total probability of the corresponding linear subspace.
} \label{fig:conicduality}
\end{figure}

By the way of an example, we estimated the maximum likelihood estimate
and its covariance matrix for the conic containing 25 points, shown in
Fig.~\ref{fig:conicdual}a, assuming i.i.d.\@ Gaussian noise in the 2D
measurements. 
To evaluate the dual
density 
at the selected location
$(x_1,x_2)$ on the 2D plane, we need to evaluate (\ref{eq:dualdensity})
as well as compute the right magnification factor. 
As we parameterised the reduced joint feature vector $\y'$ by the modified
spherical coordinates, we may construct the mapping $(x_1,x_2)
\mapsto (\rho,\phi)$ and its Jacobian to evaluate the dual
pdf using (\ref{eq:conditional}). 
The dual pdf for the estimated conic is
illustrated in Fig.~\ref{fig:conicdual}b. The contours visualise the
fact that the we have a strong belief about the true conic points
near the training data but extrapolation beyond the training data
contains a substantial risk.

\begin{figure}[t]
\begin{center}
\subfigure[]{\includegraphics[width=0.48\columnwidth]{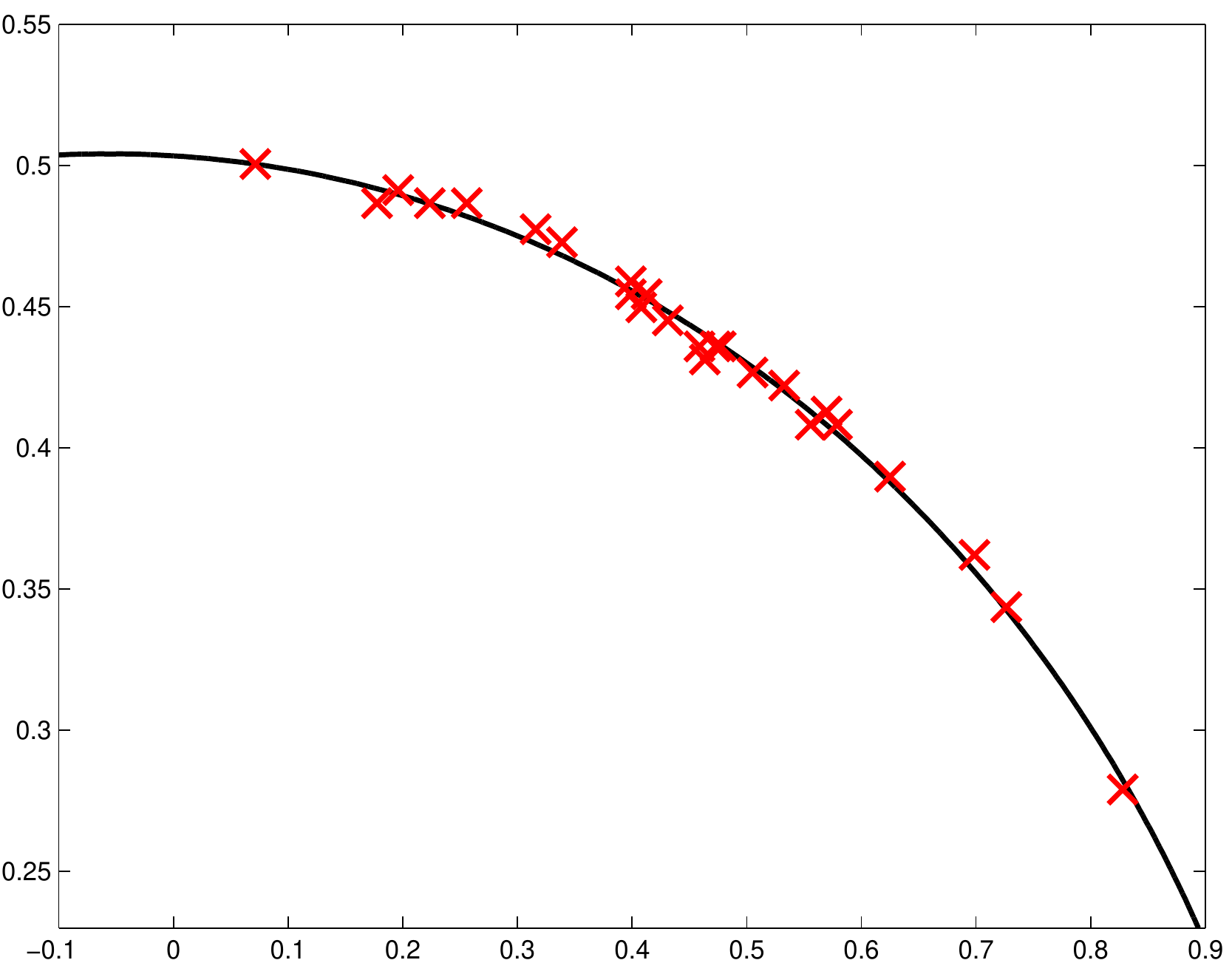}}
\subfigure[]{\includegraphics[width=0.48\columnwidth]{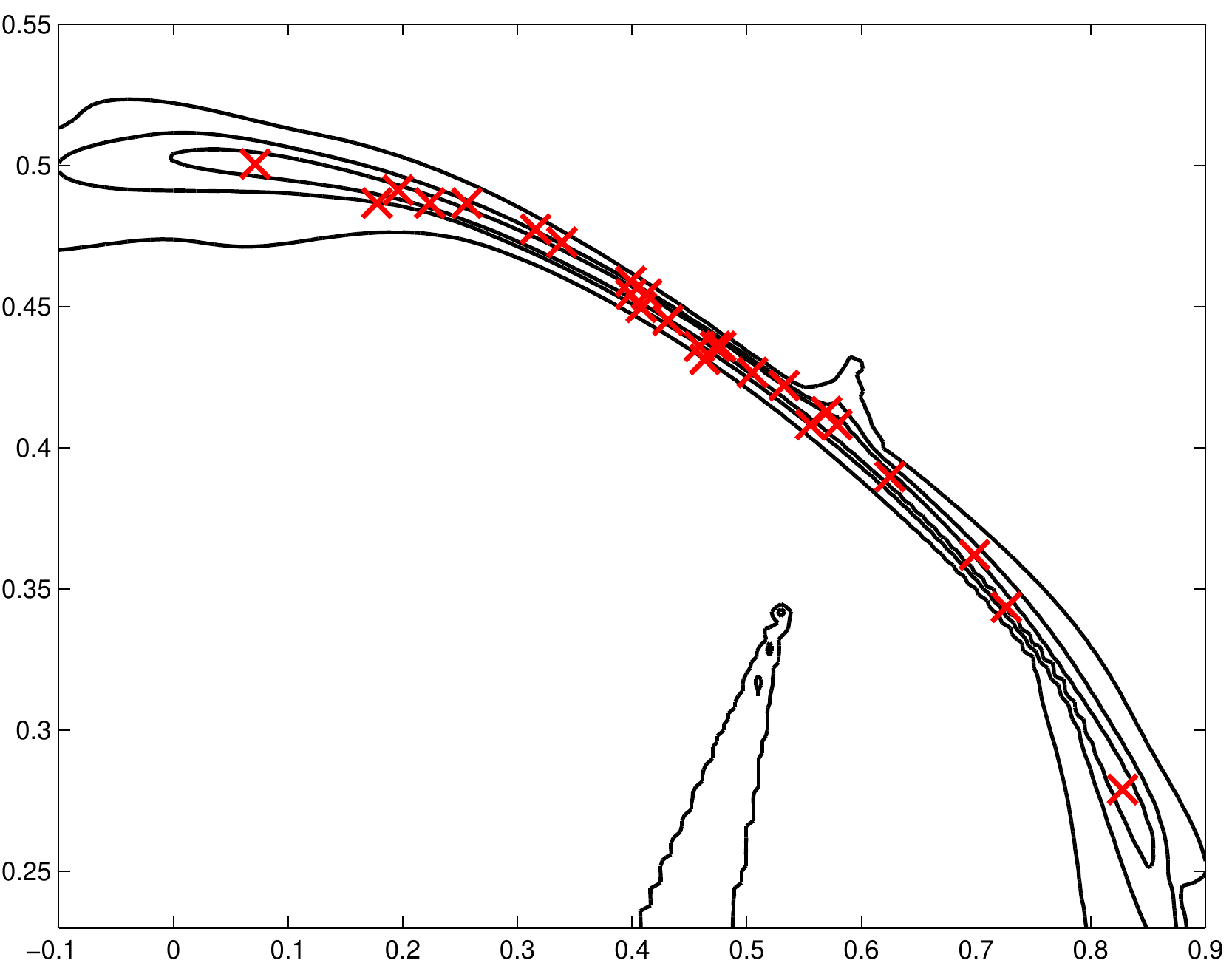}}
\end{center}
\caption{(a) The training data and the maximum likelihood estimate for the conic. (b) Contours of the dual pdf (point density) of the estimated conic characterising the current evidence where the points of the true conic locate.} \label{fig:conicdual}
\end{figure}

\begin{conic}

Using the asymptotic normality property of the MLE, we assume that
$\theta \sim N(\theta_0,\mathbf{C}_\theta)$ with the density function $p(\theta)$. Since $p(\theta)$ is Gaussian, we may analytically marginalise it over the  hyperplanes $S_\mathbf{n}$ parallel to $\mathbf{y}$ in the dual space, i.e.,
\begin{equation} \label{eq:lambda}
  p(\lambda|\mathbf{n})=\int p(\theta) \mathrm{d} S_\mathbf{n},
\end{equation}
where  $\mathbf{n}$ denotes the unit normal of the parallel hyperplanes
and $\lambda$ parameterises the (negative) signed distance of a parallel hyperplane from the
origin in the dual space, thus we define $\mathbf{y} =(\mathbf{n}^\T\ \lambda)^\T$. Since we marginalise the 6-dimensional Gaussian
distribution over the 5-dimensional subspace, $p(\lambda|\mathbf{n})$ is
one-dimensional normal distribution.

We may also derive the distribution of directions that define the
orientation of $\mathbf{n}$. This is equivalent in marginalising the
6-dimensional Gaussian distribution over the signed radius in
the modified\footnote{Instead of defining the radius parameter on the half infinite interval $[0,\infty)$, we define it on the full interval $(-\infty,\infty)$, and redefine the angle domains accordingly.} 6-dimensional spherical coordinates as 
\begin{equation}
  p(\mathbf{n})=\int_{-\infty}^\infty | \det \J \theta(\Phi) | p(\theta(\Phi)) \mathrm{d}r,
\end{equation}
where $\Phi$ denotes variables in the modified 6-dimensional spherical
coordinate system. This integral can be analytically computed, as will
be seen in next section.

The two distributions above can be combined to form the distribution for $\mathbf{y}$
\begin{equation}
     p(\mathbf{y}) \equiv p(\mathbf{n},\lambda) = p(\lambda|\mathbf{n}) p(\mathbf{n}). 
\end{equation}
The remaining problem is to transform the density of $\mathbf{y}$,
parameterised by $\mathbf{n}$ and $\lambda$, onto the
\emph{inhomogeneous} image coordinates $(u,v)$. The difficulty is that
the joint feature mapping $\mathbf{y}=\mathbf{y}(\mathbf{x})$ is not
one-to-one and onto. Therefore every $\mathbf{y}$ does not correspond
to a true image point. However, since the uncertainty of the the conic
reflect the behaviour of the training data, the variation of the conic
parameters should be small in the directions of forbidden
$\mathbf{y}$, so those can be neglected. We thus make an intermediate
transformation $\mathbf{z}(\y)=(y_1,y_2,y_3,y_4,y_4/y_5,y_4/y_6)$,
hence, $(u,v)\equiv(z_5,z_6)$.
%
%
%
%
Now, we may write
\begin{equation} \label{eq:dual}
  p(u,v)=\int_{z_1\cdots z_4}|\det \J \mathbf{z} (\y)| p(\y)\mathrm{d}z_1\cdots \mathrm{d}z_4,
\end{equation}
where we have marginalised the ``nuisance'' parameters $z_1,\ldots,z_4$ out from the distribution.

Sampling from the dual distribution (\ref{eq:dual}) is
straightforward. First we draw a sample from the six-dimensional
Gaussian $p(\theta)$ and compute the direction vector $\mathbf{n}$ of
the sample. Given $\mathbf{n}$, we then draw a sample for $\lambda$
from the 1D Gaussian (\ref{eq:lambda}). Since, $\y=(\mathbf{n},\lambda)$, we finally obtain the estimate for the point $(u,v)$, as $(y_4/y_5,y_4/y_6)$.

\end{conic}




\begin{table}[b]
  \caption{The trilinearities, adopted from \cite{Hartley00}.} \label{tab:trilinearities} 
\begin{center}
   \begin{tabular}{lr @{=} ll}
    Correspondence   & \multicolumn{2}{l}{Relation} & dof\\
\hline
    three points & $x^i {x'}^j {x''}^k \epsilon_{jqs} \epsilon_{krt} \mathcal{T}^{qr}_i$&$0_{st} $& 4\\ 
    two points, one line & $x^i {x'}^j l''_r \epsilon_{jqs} \mathcal{T}^{qr}_i$&$0_{s} $& 2\\
    one point, two lines & $x^i l'_q l''_r \mathcal{T}^{qr}_i$&$0 $& 1\\
    three lines & $l_p l'_q l''_r \epsilon^{piw} \mathcal{T}^{qr}_i$&$0^w $& 2\\
  \end{tabular}
\end{center}
\end{table}

\subsection{Probabilistic Point Transfer with an Uncertain Trifocal Tensor} \label{sec:trifocaltensor}
The geometry of three projective views is characterised by several incidence relations or trilinearities, which are collected into
Table~\ref{tab:trilinearities} \cite{Hartley00}. 
According to our
preferences, we could use any of the trilinearities to create dual
distributions or conditional dual distributions. 
As an example, we now illustrate how the trifocal point
transfer can be ``probabilised'' by constructing a dual distribution
for the trifocal tensor given estimates for the tensor and its
covariance matrix as well as a novel point match in two views. In this case, the dual distribution is the conditional position distribution of the transferred point in the third view.  
We use the nine point--point--point constraint equations (the
first relation in Table~\ref{tab:trilinearities}) of which only four equations are independent. Therefore the dimension of the affine
subspace of this example is $K=(M-1)-4=14$.

Given the point match $\m'\leftrightarrow \m''$ in the views two and three, the construction of the conditional probability density
$p(\m|\m',\m'',\mathcal{T},\C_\mathcal{T})$
is as follows. 
The trilinear relations for three points define the elements of $\Y$ in the model (\ref{eq:model2}). 
After constructing $\Y$, the whitened matrix $\Y'$ is obtained from
(\ref{eq:hyperplane2}). However, we regularised the whitening
transform by replacing $\tilde{\Lambda}$ by $\tilde{\Lambda}+\lambda
\I$, 
where $\lambda=10^{-8}$ since the number of data points was limited in our experiment 
and hence, due to overfitting of the covariance matrix, the 
smallest ones of the eighteen non-zero eigenvalues were (assuming a
non-degenerate configuration) relatively small.\footnote{The diagonal Tikhonov
regulariser was adequate here 
but a more elegant way
would have been investigating the numerical rank of the covariance
matrix and developing an automatic model selection scheme to determine the dimension of the
affine subspace. Degenerate configurations could be also handled in this
way.}
%
%
%
%
The probability density for the affine subspace, corresponding to
$\Y'$, is given by (\ref{eq:pW}) and the desired conditional probability
density values are finally obtained from (\ref{eq:conditional}), up to scale.
In short, to compute the conditional probability density value for a location of interest in the first view, the corresponding affine subspace parameters are computed from $\Y'$ and the probability density value of the corresponding affine subspace is weighted by the term containing the Jacobian of the transformation.

\begin{figure}[t]
\begin{center}
\includegraphics[width=0.49\columnwidth,trim={2cm 8cm 2cm 8cm},clip]{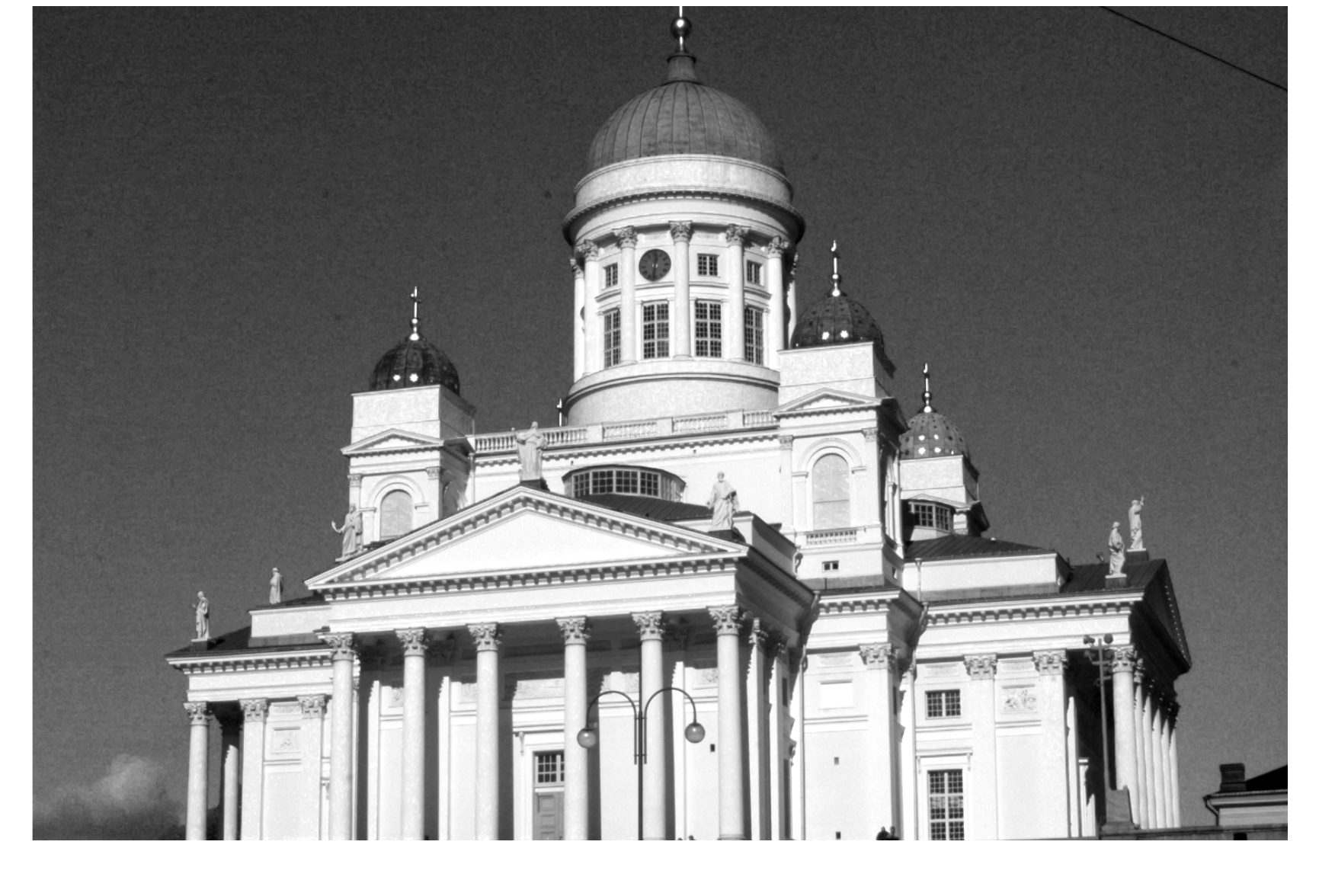}\includegraphics[width=0.49\columnwidth,trim={2cm 8cm 2cm 8cm}]{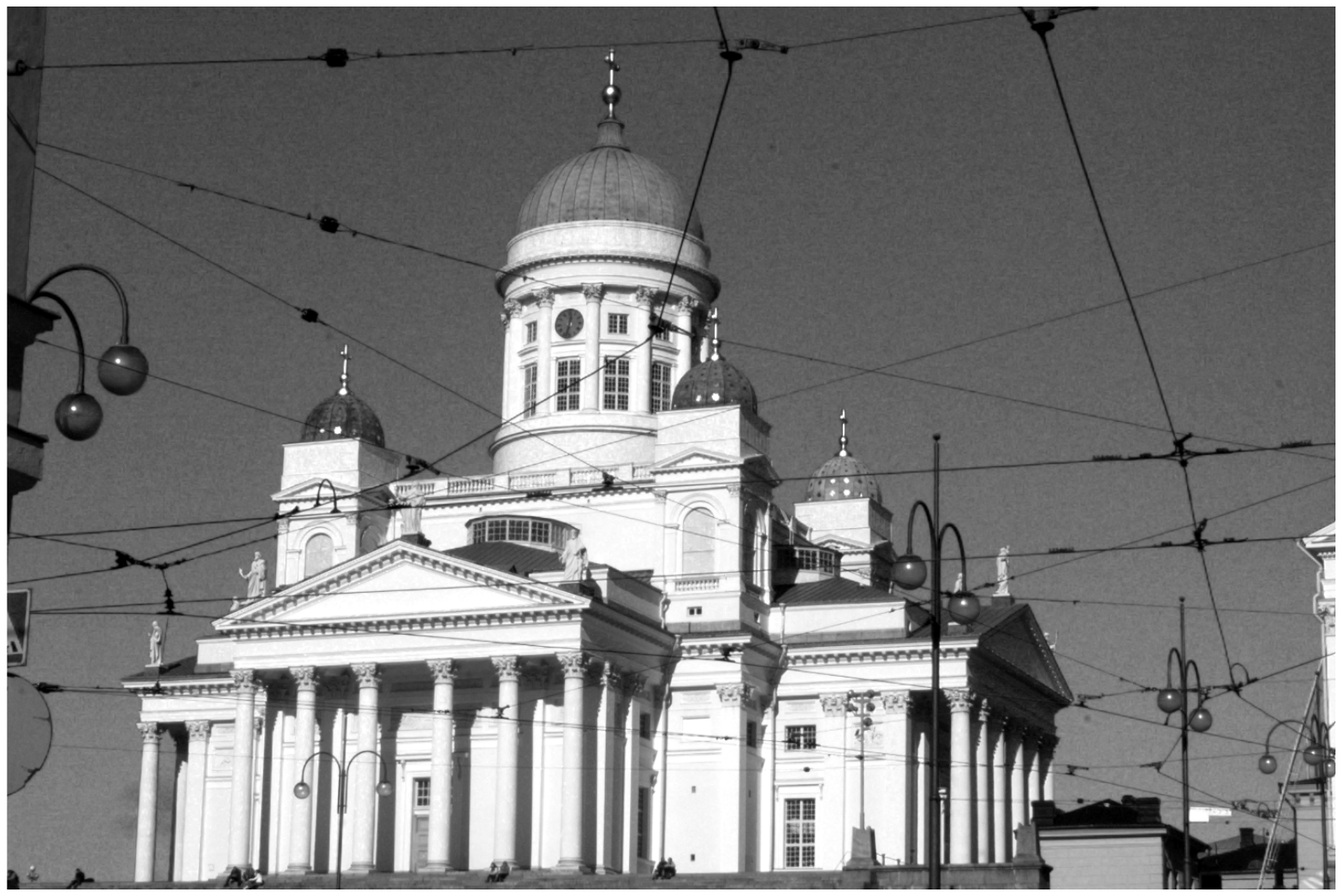}
\includegraphics[width=0.49\columnwidth,trim={2cm 8cm 2cm 8cm}]{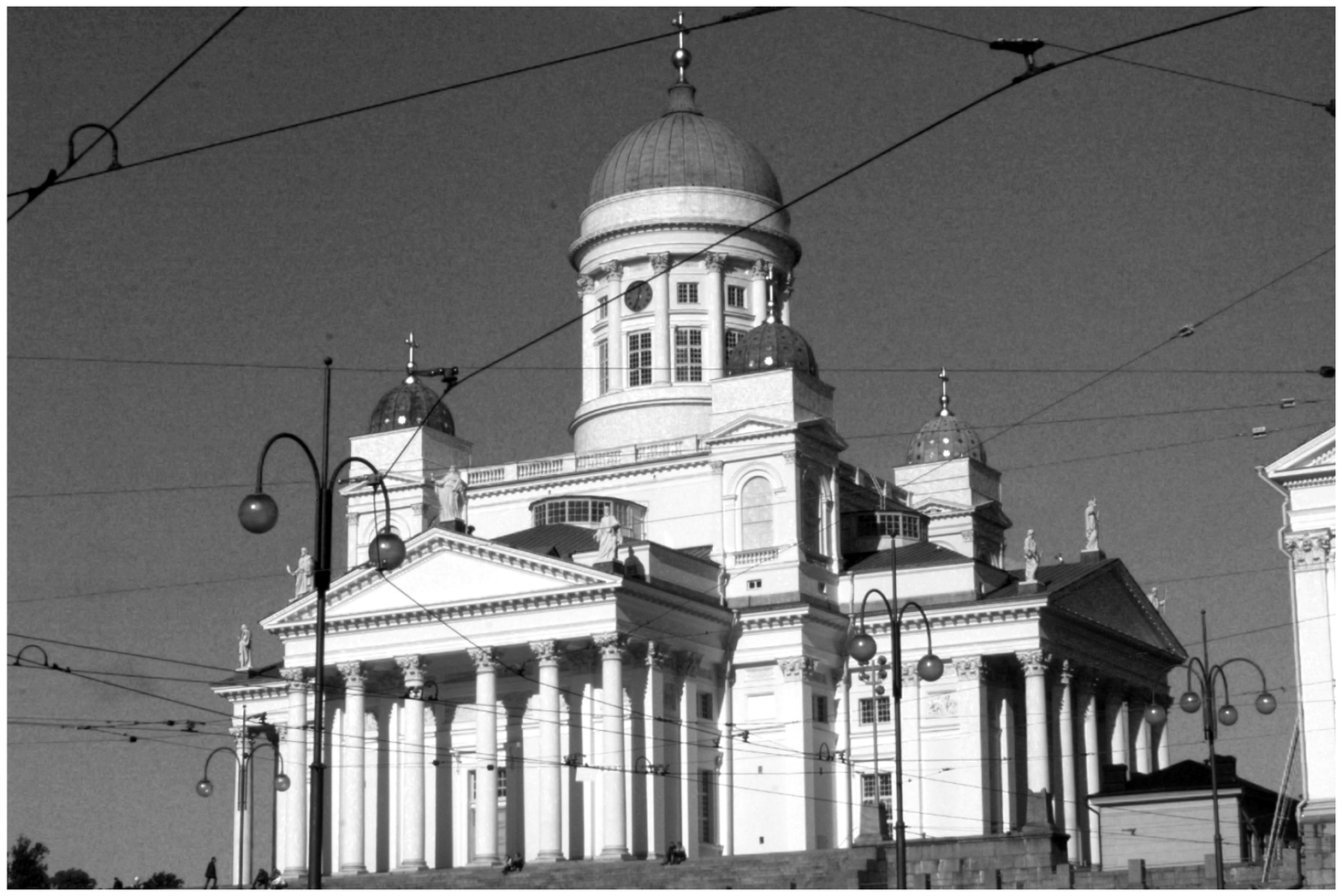}
\end{center}

\caption{Training images for estimating the trifocal tensor. 
} \label{fig:triplet}
\end{figure}

\begin{figure}[t]
\begin{center}
\subfigure[First view]{\includegraphics[width=0.32\columnwidth]{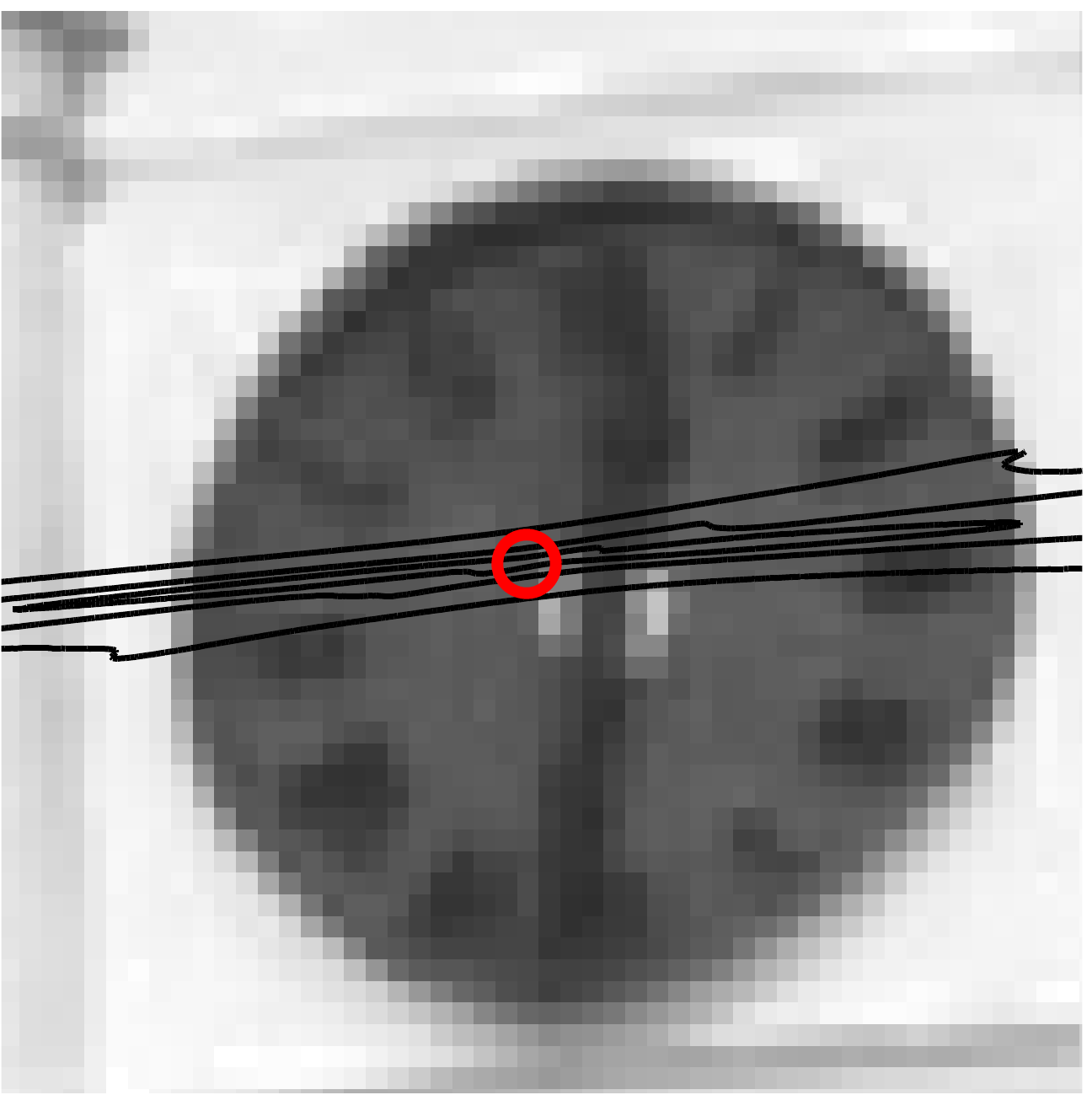}} 
\subfigure[Second view]{\includegraphics[width=0.32\columnwidth]{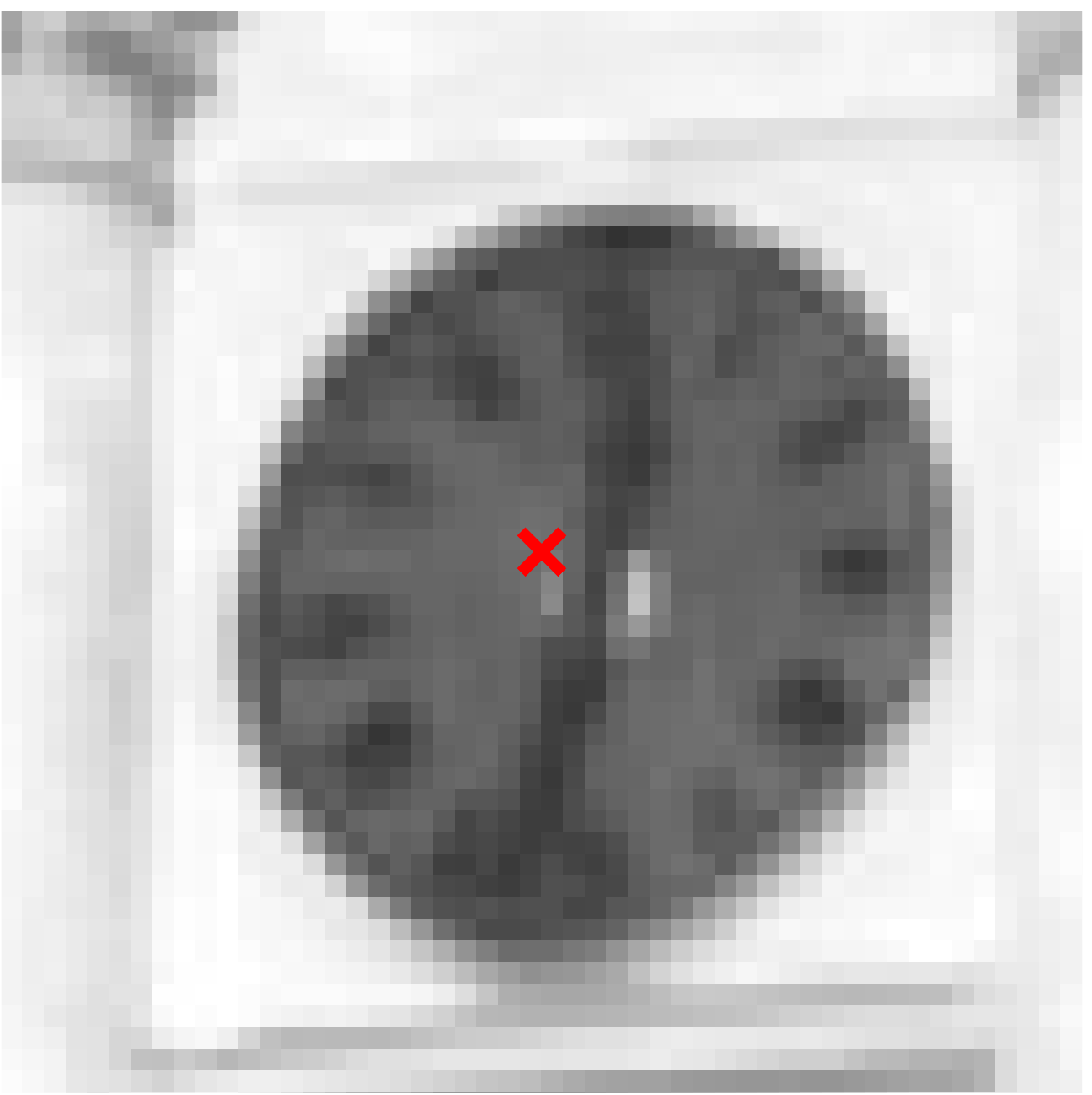}}
\subfigure[Third view]{\includegraphics[width=0.32\columnwidth]{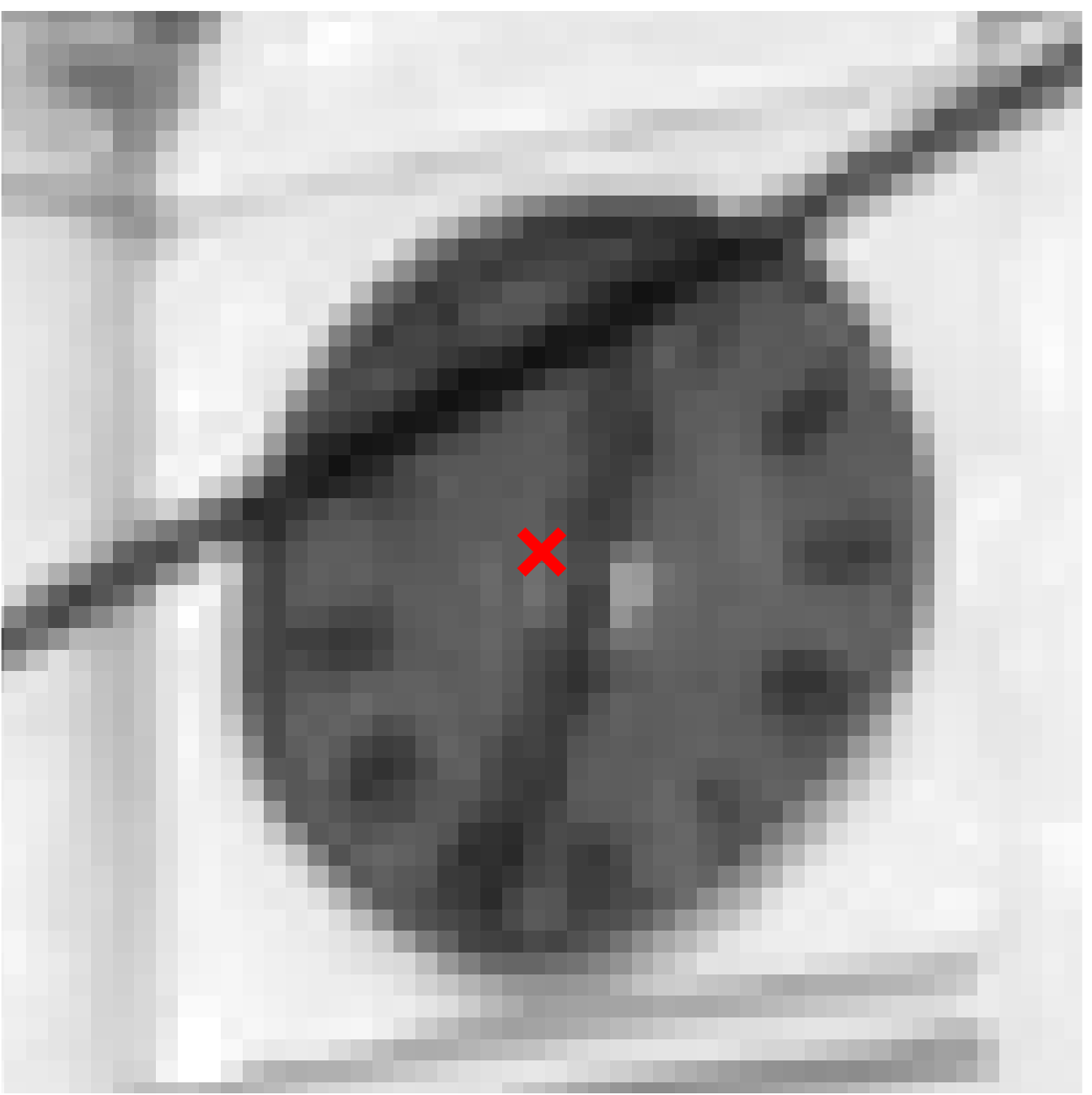}}\\
\subfigure[Contours]{\includegraphics[width=0.32\columnwidth]{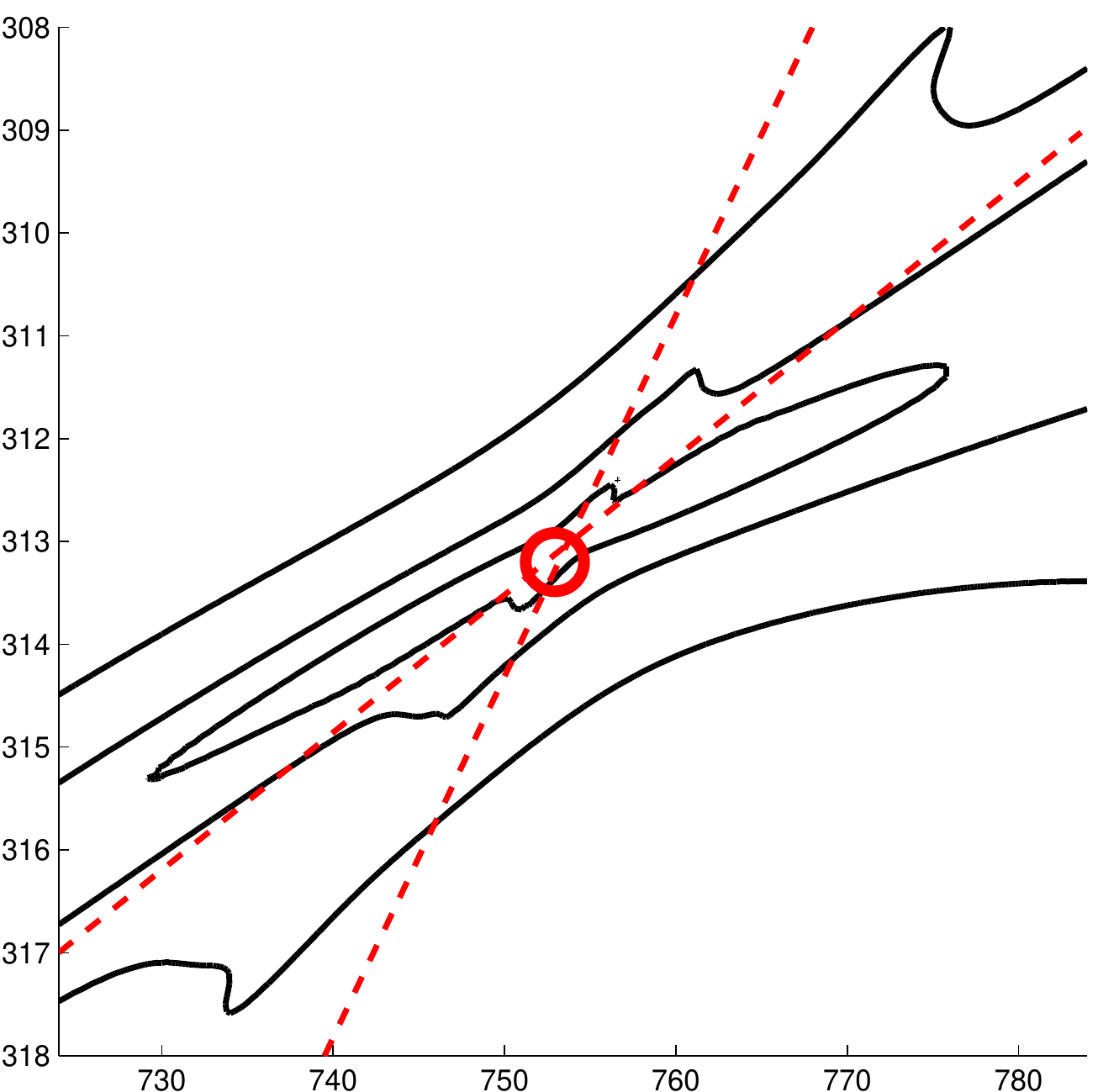}} 
\end{center}
\caption{Probabilistic point transfer. 
(a) Three equi-probability contours of the point transfer pdf in the
first view, at the levels $10^{-1}$, $10^{-2}$, and $10^{-3}$ times
the maximum value, conditioned on the two-view point match in the
views (b) two and (c) three. The circle indicates the transferred
point using the (deterministic) point transfer \cite{Hartley00} 
with
only the ML estimate for the trifocal tensor, and the two points in
the other views; (d) the same contours shown after scaling of the
y-axis and superimposing the ML epipolar lines (dashed) corresponding
to the points given in the second and third view. The contours are
closed curves surrounding the most likely match locations and
illustrate the fact there the trifocal transfer is more than the
epipolar transfer from the views two and three.} \label{fig:trifocaltransfer}
\end{figure}

To illustrate the conditional point transfer density, we estimated the
maximum likelihood trifocal tensor and its covariance matrix from 274
outlier free point correspondences for the image triplet shown in
Fig.~\ref{fig:triplet} assuming i.i.d. Gaussian noise in the
measurements \cite{Hartley00}. Then, by using a novel point
correspondence in the views two (Fig.~\ref{fig:trifocaltransfer}b) and three (Fig.~\ref{fig:trifocaltransfer}c), we computed the conditional
probability density in the first view, as reported above. To visualise
the shape of the pdf, we selected, three pdf values at the levels
$10^{-1}$, $10^{-2}$, and $10^{-3}$ times the maximum value
and
show the corresponding contours in Fig.~\ref{fig:trifocaltransfer}a and \ref{fig:trifocaltransfer}d. It
can be seen that the transferred density is
non-Gaussian while it indicates the feasible locations for the
correspondence. The pdf has its maximum close to the point
where the ML epipolar lines meet whereas the local shape of the peak is oriented towards the mean axis of the two epipolar lines. The pdf shape also seems to illustrate the well known fact that the trifocal constraint is
more versatile than the mere epipolar geometries between the three views.


\section{Discussion}

The dual distributions are a tool for Bayesian inference with uncertain multilinear models, used e.g.\@ in geometric
image analysis. According to the Bayesian paradigm, we handle
geometric relationships as probability distributions in contrast to
the traditional approaches, which are based on deterministic estimates.
Moreover, the theory provides tools for propagating the
geometric uncertainty information to the dual variables. This approach
helps especially in the
cases when there is no explicit form for the mapping but 
a multilinear incidence relation which indirectly connects the parameter space and
the observations.

As in the point--line case there are certain points and lines with a special role \cite{Brandt08}, likewise there are \emph{special subspaces} in the general case to be investigated in future. They arise from the affine subspaces spanned by the eigenvectors of the covariance matrix---the deterministic trifocal point transfer, for instance, should be constructed by identifying the most likely affine subspace in the Gaussian model using the point--point--point incidence relation. In addition, there are many other ways to utilise the special subspaces arising from the multiple constraint equations. However, one should be careful in the interpretations as covariance estimates are not generally invariant to the selection of the coordinate frame. 


This work is fundamentally a study of the generalised Radon transform for probability measures. Since the transform is based on integration over all the affine subspaces, a suitable parameterisation for affine subspaces is required. A natural algebraic choice 
would be the Grassmannian parameterisation, and it is likely that it 
would provide certain algebraic benefits that are to be investigated in future. For instance, one should study the transforms with respect to a natural Riemannian metric, invariant to motions of the coordinate system, to make the approach invariant to Gauge transforms. Proceeding in this direction would clarify the role of the dual distributions in between Santalo's and Gelfand's schools of integral geometry.


As far as the numerical development is concerned, more efficient numerical tools would be advantageous for geometric inference.  The study of special subspaces and efficient sampling from the dual distributions in the general case are additionally to be investigated in future. Likewise, the study of marginal distributions, and proper handling of nonlinear geometric constraint manifolds are part of the future work.



\section{Conclusions} \label{sec:conclusions}

In this paper we have shown how the parameter distributions of
multilinear models can be dualised.  
The proposed approach provides means for a pure statistical treatment
of multilinear relations where the uncertainty of the geometric entity is taken into account. The dual distributions are closely connected to integral geometry and have analytic forms with relatively small assumptions.
%
%
%
%
%
%
%
%
%
%
We demonstrated the applicability of the theory by characterising the confidence of estimated conics and constructed the probabilistic trifocal point
transfer by using an uncertain trifocal tensor with its estimated covariance
matrix.
In future, an interesting research direction is to investigate
the dual distributions from the view point of integral geometry to
deeply understand 
the theoretical connections 
as well as to
develop additional numerical methods for applications.

%


%
%



\bibliographystyle{IEEEtran}
\bibliography{cvpr2008,emt}


\begin{appendix}

\section*{A. Modified Spherical Coordinates in $\R^N$} \label{app:ndim}

In this paper, we use modified spherical coordinates in
$N$-dimensional space. That is, we assign a sign for the radius in the
conventional $N$-dimensional spherical coordinates \cite{Weeks85} and parameterise the directions using points only
on the other half of the unit hypersphere. This parameterisation is
natural for one-dimensional subspaces in $\R^N$ as the directions then
uniquely parameterise the linear subspaces almost everywhere and the
(signed) radial parameter parameterises the points in the subspace.

Assume that\footnote{If $N=1$ or $N=2$ the construction is similar but these cases are omitted here for simplicity.} $N\geq 3$ and let $\rho \in \R$ be the radial coordinate and $\phi_1,\phi_2,\ldots,\phi_{N-1}$ be the angular coordinates so that $\phi_1$ takes values between 0 and $\pi/2$, $\phi_2,\ldots,\phi_{N-2}$ are between $0$ and $\pi$ and $\phi_{N-1}$ is between $0$ and $2\pi$.
The Cartesian coordinates $x_i$, $i=1,\ldots,N$ are then defined as  
\begin{equation}
\begin{split}
  x_1&=\rho \cos(\phi_1), \\
  x_{i}&
=\rho \cos(\phi_{i}) \prod_{k=1}^{i-1} \sin(\phi_k), \quad i=2,3,\ldots,N-1,\\
  x_N&=\rho \prod_{i=1}^{N-1} \sin(\phi_i).
\end{split}
\end{equation}
The inverse transform is 
\begin{equation}
\begin{split}
\rho &=  \sign (x_1) \|\mathbf{x} \|, \quad 
%
\phi_{1}=\arctan\left(\sqrt{\frac{\|\x\|^2}{x_1^2}-1}\right), \\
\phi_{i}&=\arctan\left(\frac{\sqrt{\sum_{k={i+1}}^{N}{x_k^2}}}{\sign(x_1) x_{i}}\right),  \quad i=2,3,\ldots N-2, \\
\phi_{N-1}&=
\arctan({\sign(x_1)x_N},{\sign(x_1)x_{N-1}}),
%
\end{split}
\end{equation}
%
and the volume element is obtained as 
\begin{equation}
\begin{split}
&\left| \det \frac{\partial(x_1,x_2,\ldots,x_N)}{\partial(\rho,\phi_1,\phi_2,\ldots,\phi_{N-1})} \right|
\dd \rho \dd \phi_1 \dd \phi_2 \cdots \dd\phi_{N-1}\\
&\ = | \rho|^{N-1} \left(\prod_{k=1}^{N-2} \sin^{N-k-1}(\phi_k) \right) \dd \rho \dd \phi_1 \dd\phi_2 \cdots \dd\phi_{N-1}.
\end{split}
\end{equation}

\section*{B. Proof of Lemma 1} \label{sec:lemmaproof}

\begin{proof} It is well known that an orthogonal projection matrix $\proj$ may be decomposed as $\proj=\mathbf{U}\mathbf{U}^\T$ where the columns of $\mathbf{U}$ form an orthogonal basis of the range of $\proj$. We hence form an orthonormal basis for the range from the basis $\mathbf{p}_i=\mathbf{P}\mathbf{e}_i$, $i=1,\ldots,K$ by using the Gram--Schmidt orthonormalisation procedure. Now, let 
\begin{equation} \label{eq:orthogonalisation}
\begin{split}
  \mathbf{u}_1&=\frac{\proj \e_1}{\| \proj \e_1\|}, \quad 
  \mathbf{u}_2
=\frac{\proj \e_2-( \mathbf{u}_1^\T\proj \e_2)\mathbf{u}_1}{\| \proj \e_2- ( \mathbf{u}_1^\T\proj \e_2)\mathbf{u}_1\|},\quad 
\ldots,\\
\mathbf{u}_{K}
&=\frac{\proj \e_{K} - \sum_{i=1}^{K-1}( \mathbf{u}_i^\T\proj \e_K)\mathbf{u}_i}{\| \proj \e_{K}- \sum_{i=1}^{K-1}( \mathbf{u}_i^\T\proj \e_K)\mathbf{u}_i\|}.
\end{split}
\end{equation}
and let $\Lo=(\mathbf{u}_1\ \mathbf{u}_2\ \ldots\ \mathbf{u}_{K})$, i.e., the columns of $\Lo$ form an orthonormal basis for the range of $\proj$, and $\proj=\Lo \Lo^\T$. 

We need to show that $\Lo$ is additionally a  lower triangular matrix which has a strictly positive diagonal. We note that 
\begin{equation}
\begin{split}
 \mathbf{u}_1&=\frac{\proj \e_1}{\| \proj \e_1\|} =
\frac{ \Lo \Lo^\T \e_1}{\|  \Lo \Lo^\T \e_1\|} =
\frac{\sum_{k=1}^K \mathbf{u}_k \mathbf{u}_k^\T \e_1}{\|  \sum_{k=1}^K \mathbf{u}_k \mathbf{u}_k^\T \e_1\|}\\
&
=\frac{u_{11} \mathbf{u}_1  +\overset{=0}{\overbrace{\sum_{k=2}^K {u}_{1k} \mathbf{u}_k}}}{\|  \sum_{k=1}^K {u}_{1k} \mathbf{u}_k\|},
\end{split}
\end{equation}
where the second term in the nominator is zero because the vectors $\mathbf{u}_k$, $k=1,2,\ldots,K$ are linearly independent by assumption. It follows that $u_{1k}=0$, $k=2,\ldots,K$. Thus, $u_{11}/|u_{11}|=1 \implies u_{11}>0$. Similarly, 
\begin{equation}
 \mathbf{u}_2
=\frac{u_{22} \mathbf{u}_2  \overset{=0}{\overbrace{-( \mathbf{u}_1^\T\proj \e_2)\mathbf{u}_1+\sum_{k\neq 2}^K {u}_{2k} \mathbf{u}_k}}}{\|  \sum_{k=1}^K {u}_{2k} \mathbf{u}_k\|},
%
\end{equation}
that implies $u_{2k}=0$, $k=3,\ldots,K$ and $u_{22}>0$. Likewise, $u_{lk}=0$, and $u_{ll}>0$ when $k=l+1,\ldots,K$, hence, $\Lo$ is a lower triangular matrix with strictly positive diagonal. 

We still need to show that the representation is unique. Let us
assume the contrary, i.e., there are two different lower triangular
matrices $\Lo_1=(\mathbf{u}_1\ \mathbf{u}_2\  \cdots\  \mathbf{u}_K)$ and $\Lo_2=(\mathbf{v}_1\ \mathbf{v}_2\ \cdots\ \mathbf{v}_K)$ which have the the indicated
properties. Then $\proj=\Lo_1 \Lo_1^\T=\Lo_2 \Lo_2^\T$ and 
\begin{equation}
{\proj \e_1}=u_{11}\mathbf{u}_1=v_{11}\mathbf{v}_1. 
\end{equation}
Since $\mathbf{u}_1$ and $\mathbf{v}_1$ are unit vectors, it follows that $|u_{11}|=|v_{11}|$. The diagonal elements are additionally strictly positive so $u_{11}=v_{11}$ that implies $\mathbf{u}_1=\mathbf{v}_1$.
Similarly, we see that $\mathbf{u}_l=\mathbf{v}_l$, when $l=2,\ldots,K$, i.e., $\Lo_1\equiv\Lo_2$ that contradicts the assumption, hence, $\Lo$ must be unique. 
\end{proof}


\end{appendix}

\end{document}